\documentclass{article} 
\usepackage{iclr2026_conference,times}

\usepackage{amsmath,amsfonts,bm}

\def\eqref#1{equation~\ref{#1}}

\def\1{\bm{1}}

\DeclareMathAlphabet{\mathsfit}{\encodingdefault}{\sfdefault}{m}{sl}
\SetMathAlphabet{\mathsfit}{bold}{\encodingdefault}{\sfdefault}{bx}{n}

\newcommand{\R}{\mathbb{R}}

\newcommand{\Var}{\mathrm{Var}}

\usepackage{hyperref}
\usepackage{url}

\usepackage{subcaption} 
\usepackage{float}

\usepackage[utf8]{inputenc}
\usepackage[T1]{fontenc}
\usepackage{hyperref}
\usepackage{url}
\usepackage{booktabs}
\usepackage[table,svgnames]{xcolor}
\usepackage{amsmath, amssymb, amsthm, bm}
\usepackage{algorithm}
\usepackage{algpseudocode}
\usepackage{ulem}

\theoremstyle{plain}
\newtheorem{theorem}{Theorem}
\newtheorem{lemma}{Lemma}

\newtheorem{corollary}{Corollary}

\theoremstyle{remark}
\newtheorem*{remark}{Remark}

\theoremstyle{definition}
\newtheorem{definition}{Definition}
\usepackage{nicefrac}
\usepackage{microtype}
\usepackage{xcolor}
\usepackage{graphicx}
\usepackage{subcaption}
\usepackage{wrapfig}
\usepackage{multirow}
\usepackage{tcolorbox}
\newcommand{\EE}{\mathbb{E}}
\renewcommand{\Var}{\mathrm{Var}}

\usepackage{makecell}
\usepackage{enumitem}
\usepackage{array}
\usepackage{colortbl}

\usepackage{pgfplots}
\pgfplotsset{compat=1.18}

\usepackage{algorithm}
\usepackage{algpseudocode}

\usepackage{tikz}

\usepackage{pifont}

\title{Towards Fast LLM Fine-tuning through Zeroth-Order Optimization with Projected Gradient-Aligned Perturbations}

\author{
  Zhendong Mi\textsuperscript{1}, 
  Qitao Tan\textsuperscript{2}, 
  Grace Li Zhang\textsuperscript{3}, 
  Zhaozhuo Xu\textsuperscript{1}, 
  \textbf{Geng Yuan}\textsuperscript{2}, 
  \textbf{Shaoyi Huang}\textsuperscript{1}\thanks{Corresponding author.} \\[4pt]
  \textsuperscript{1}Stevens Institute of Technology
  \textsuperscript{2}University of Georgia  
  \textsuperscript{3}Technical University of Darmstadt \quad \\
  \texttt{\small \{zmi2, zxu79, shuang59\}@stevens.edu}, 
   \texttt{\small \{qitaotan, geng.yuan\}@uga.edu}, \\
   \texttt{\small \{grace.zhang\}@tu-darmstadt.de}
}

\iclrfinalcopy 
\begin{document}

\maketitle

\begin{abstract}
Fine-tuning large language models (LLMs) using zeroth-order (ZO) optimization has emerged as a promising alternative to traditional gradient-based methods due to its reduced memory footprint requirement. However, existing ZO methods suffer from high variance in gradient estimation, leading to slow convergence and suboptimal performance on large-scale models. In this work, we propose P-GAP,
a fast LLM fine-tuning approach through zeroth-order optimization with \underline{P}rojected \underline{G}radient-\underline{A}ligned \underline{P}erturbations.
Specifically, 
%
we first estimate a low-dimensional gradient space and then align perturbations in projected gradients' direction within the space. 
This approach enables reduced the number of perturbed parameters and decreased variance, therefore accelerated convergence for LLM fine-tuning. Experiments on LLMs show that P-GAP consistently surpasses the baselines, achieving up to 6\% increase in accuracy on classification tasks and up to 12\% higher accuracy on generation tasks, with up to about 81\% less training iterations and 70\% less GPU hours.
These results demonstrate that P-GAP enables fast, scalable, and resource-efficient ZO LLM fine-tuning.
\end{abstract}

\section{Introduction}

Fine-tuning (FT) large language models (LLMs)~\citep{hu2021lora, dettmers2023qlora, gu2021efficient} for specific tasks or datasets has become a common practice in modern machine learning. However, as model size and complexity scale, fine-tuning incurs substantial memory overhead, which severely limits its scalability and makes it inaccessible to users with constrained computational resources~\citep{tan2025harmony, zhao2024second}. To alleviate this issue, parameter-efficient fine-tuning (PEFT) methods have been proposed \citep{ li2021prefix, dettmers2023qlora, zhao2024galore}, which update only a small subset of parameters while freezing the majority of the model. These approaches drastically reduce GPU memory footprint and storage cost while achieving performance comparable to full FT. However, despite their efficiency, PEFT methods still require computing and storing full gradients and intermediate activations during backpropagation, which introduces significant memory overhead \citep{malladi2023fine, liu2024sparse}.

To address the challenge, zeroth-order (ZO) optimization has emerged as a promising solution \citep{zhang2024revisiting, malladi2023fine, mi2025kerzookernelfunctioninformed}, which estimates gradients using only forward passes. By leveraging randomized perturbations to approximate gradient directions, ZO completely removes the need to store large gradient tensors and intermediate activations, which substantially reduces memory usage. This advantage makes ZO especially appealing for extremely large models where backward passes dominate GPU memory consumption. When combined with parameter-efficient strategies, ZO-based fine-tuning offers a scalable and resource-friendly framework for adapting high-capacity models under tight memory constraints while maintaining competitive performance \citep{tan2025harmony}.
Despite the advantages of zeroth-order optimization in reducing memory overhead, these benefits often come at the expense of longer computational time (e.g., GPU hours) and decreased accuracy compared to first-order approaches \citep{li2024addaxutilizingzerothordergradients, gautam2024variancereducedzerothordermethodsfinetuning}. 

%
%

Existing works show that variance in the zeroth-order gradient estimation, attributing to the random perturbations, can be a factor for the longer computational time \cite{chen2024enhancing, park2025unraveling}. The larger variance in the estimation of the ZO gradient can lead to suboptimal accuracy and slower
\begin{wrapfigure}[17]{r}{0.6\textwidth}
\centering
\includegraphics[width=0.95\linewidth]{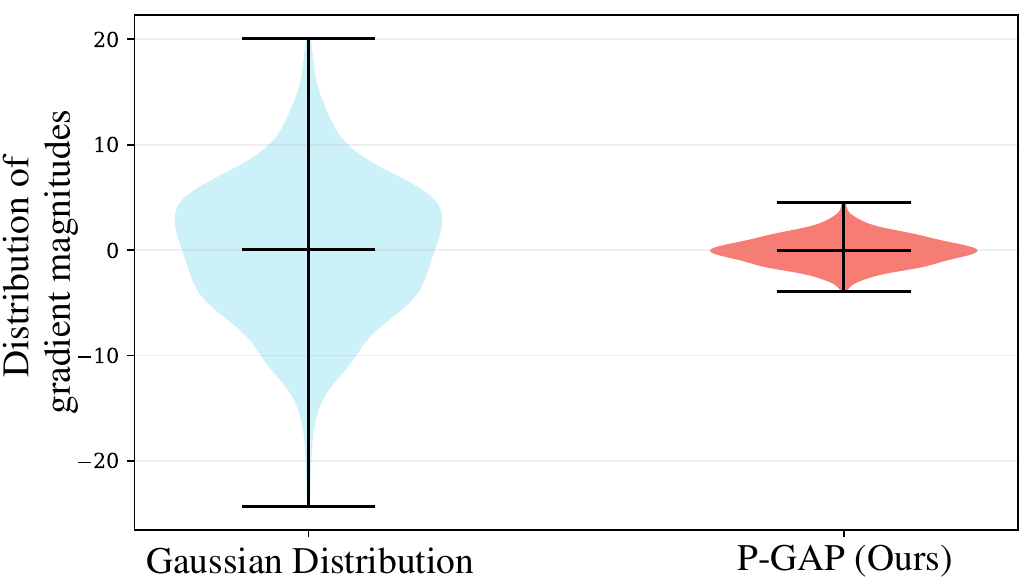}
\vspace{-0.1in}
\captionsetup{justification=raggedright,singlelinecheck=false}
\caption{Estimation of directional derivative magnitudes on the $V$ matrix from the first Transformer layer of the OPT-2.7B model, using perturbations sampled from a standard Gaussian distribution and from P-GAP}
\label{fig:distribution}
\end{wrapfigure}
convergence rates compared to first-order methods, making ZO-based fine-tuning less stable and resource-intensive \citep{kornilov2023accelerated, zhang2024revisiting, lobanov2023accelerated}. 
Existing works in LLM fine-tuning such as \citep{ohta2020sparseperturbationsimprovedconvergence} and \citep{malladi2023fine} aim to reduce the variance via increasing the number of perturbations, which will lead to prolong training time.
%
%
%
%


Inspired by \citep{ma2025revisiting,kozak2023zeroth} which find 
anisotropic perturbations (i.e., the magnitude of perturbations is larger along certain directions and is smaller along others, rather than being uniform in all directions) can potentially help relieve the variance issue in ZO optimization theoretically, we raise the following question: 

\textbf{Q1: \textit{For LLM finetuning on larger-scale models, can we find the proper perturbation directions, thereby reducing the variance of ZO gradient estimation and finally accelerating convergence with negligible accuracy loss?}}

Inspired by \citep{wang2018stochastic, zhang2024dpzero, yue2023zeroth}, which identify that perturbing the full parameter space can further amplify the variance in gradient estimation, as the variance scales proportionally with the parameter dimension $d$, we naturally pose a research question:

\textbf{Q2: \textit{Can we further reduce variance in gradient estimation by 
decreasing the parameter space 
that require perturbation-based gradient estimation?}}





To answer the two questions, we propose a fast LLM finetuing approach through zeroth-order optimization with \underline{P}rojected \underline{G}radient-\underline{A}ligned \underline{P}erturbations (P-GAP), which reduces the variance in gradient estimation of ZO updates via low-dimensional perturbations that are aligned with the gradient direction in the subspace of gradient, thereby achieving faster convergence.
Figure \ref{fig:distribution} shows the magnitude of gradient estimation 
on the attentions matrix $V$ in the first Transformer layer of the OPT-2.7B model, based on perturbations sampled from a standard Gaussian distribution and from P-GAP. 
It can be observed that the value of estimated gradients are more stable and less dispersed for P-GAP, which indicates a smaller variance. Our contributions can be summarized as follows:

\begin{itemize}[leftmargin=*]
    \item We propose a novel ZO-based LLM fine-tuning framework, P-GAP, which estimates a low-dimensional gradient space and aligns perturbations in projected gradients' direction within the space. 
    This design can not only allow the perturbation aligned in the most informative direction but also effectively reduce the dimensionality of gradient estimation, therefore reducing variance and accelerate convergence.

    

    \item We provide theoretical analysis on that the variance of ZO gradient estimation linearly increases with the dimension of weight matrix which need perturbations for gradient estimation in LLMs, 
    and further show that
    P-GAP can reduce the variance with the proposed low-dimensional gradient space estimation. Moreover, we provide the convergence analysis of P-GAP.
    
    \item We conduct extensive experiments on both encoder-only models (e.g., RoBERTa-large) and decoder-based LLMs (e.g., OPT-2.7B/6.7B and LLaMA-3-3B/8B). Results show that P-GAP achieves up to 6\% accuracy gains over the baselines, while achieving 5.2$\times$ speedup in training and more than 61 minutes less wall-clock time.
\end{itemize}

\section{Preliminaries}

\textbf{Notations.} In this paper, all of the non-bold letters (including Latin letters and Greek letters) indicate the scalar such as $\delta$ and $K$. All of the lower-case letters which is bold indicate a column vector such as $\bm{u}$ and all of the upper-case bold letters such as $\bm{V}$ indicate a matrix. 
A $d$-dimensional multivariate Gaussian distribution is denoted by 
$\mathcal{N}(\bm{\mu}, \bm{\Sigma})$,
where $\bm{\mu} \in \mathbb{R}^d$ is the mean vector and 
$\bm{\Sigma} \in \mathbb{R}^{d \times d}$ is the covariance matrix. 
We use $\mathbb{E}[\cdot]$ to represent the expected value of a variable and use 
$\mathrm{Var}[\cdot]$ to represent the variance of a variable.  
$\mathrm{vec}(\bm{W})$ indicates that we flatten the matrix $\bm{W}$ by stacking its columns vertically to change it into a column vector. 
$\|\bm{x}\|_p = (\sum_{i=1}^n x_i^p)^\frac{1}{p}$ indicates the $\ell_{p}$-norm of a vector $\bm{x}$ and we use $||\bm{x}||$ to denote the $\ell_{2}$-norm of a vector $\bm{x}$.
$\|\bm{U}\|_F = \sqrt{\langle \bm{U}, \bm{U} \rangle}$ denotes the Frobenius norm of a matrix $\bm{U}$ and we will call it F-norm in the paper for simplicity.
$\mathcal{C}_L^{s,p}(\mathcal{S})$ denotes the collection of functions defined on the set 
$S$ that are 
$s$-times continuously differentiable, and whose 
$p$-th order derivatives are 
$L$-Lipschitz continuous. 
$\widehat{\nabla}$ indicates the estimation of gradient and $\nabla$ indicates the true gradient. $I$ indicates the identity matrix or vector.



\textbf{Zeroth-order Optimization for LLMs.}
%
%
Consider a large language model with parameters $\bm{\theta} \in \mathbb{R}^d$ and loss function $\mathcal{L}$. 
At iteration step $t$, zeroth-order optimization estimates the gradient on a mini-batch datasets $\mathcal{B}_t$ by perturbing $\bm{\theta}_t$ along random directions. 
Specifically, if we choose to use Gaussian distribution as perturbations, then we can get $\bm{u} \sim \mathcal{N}(0,I_d)$ and $\mathcal{N}(0,I_d)$ is the standard Gaussian distribution.
Given a perturbation scale $\epsilon > 0$, the two-point gradient estimator is
\begin{equation}
\small
    \widehat{\nabla} \mathcal{L}(\bm{\theta}_t; \mathcal{B}_t) 
    = \frac{\mathcal{L}(\bm{\theta}_t + \epsilon \bm{u}; \mathcal{B}_t) - 
    \mathcal{L}(\bm{\theta}_t - \epsilon \bm{u}; \mathcal{B}_t)}{2\epsilon}\,\bm{u}
    \label{eq:zo_estimator}
\end{equation}
where $\widehat{\nabla}$ in Equation \ref{eq:zo_estimator} indicates the estimated gradients. To reduce estimator variance, one may average over $n$ independent perturbations $\{\bm{u}_i\}_{i=1}^n$:
\begin{equation}
\small
    \widehat{\nabla} \mathcal{L}(\bm{\theta}_t; \mathcal{B}_t) 
    = \frac{1}{n}\sum_{i=1}^n 
    \left[
    \frac{\mathcal{L}(\bm{\theta}_t + \epsilon \bm{u}_i; \mathcal{B}_t) - 
    \mathcal{L}(\bm{\theta}_t - \epsilon \bm{u}_i; \mathcal{B}_t)}{2\epsilon}\,\bm{u}_i
    \right]
    \label{eq:zo_estimator_multi}
\end{equation}
Finally, given the learning rate $\eta$ and estimated gradients in Equation \ref{eq:zo_estimator_multi}, the parameter update follows the standard SGD form:
\begin{equation}
\small
    \bm{\theta}_{t+1} = \bm{\theta}_t - \eta\, \widehat{\nabla} \mathcal{L}(\bm{\theta}_t; \mathcal{B}_t).
    \label{eq:zo_update}
\end{equation}
%

\section{Methodology}

In this section, we first clarify the remaining problems in existing  zeroth-order optimization frameworks and put up the motivation for our proposed method. Then, we will elaborate on our proposed \textbf{P-GAP}, which performs ZO updates with low-dimensional perturbations that are aligned with the gradient direction in the subspace of the gradient for variance reduction.
Intuitively, our pipeline begins by obtaining an approximate gradient matrix, which can be expressed as the product of low-rank frame matrices and a coefficient matrix. Within this lower-dimensional space spanned by the frame matrices, Gaussian perturbations may be selected arbitrarily without restriction; however, we hope that they are constrained to be aligned with the directions defined by the gradient’s coefficient matrix (i.e. the hyperplane defined by the low-dimension gradient's coefficient matrix). After correction, the perturbation itself can also be represented as a corrected coefficient matrix, which, when multiplied with the frame matrices, yields the final perturbation in the original high-dimensional parameter space. In other words, we allow perturbations to be chosen freely within the linear subspace spanned by 
 low-rank frame matrice, but enforce that they remain parallel to the hyperplane determined by the gradient’s coefficient matrix.

\subsection{Projected Gradient-Aligned Perturbation}


Inspired by \citep{ma2025revisiting}, we adopt the idea of projecting the sampled random perturbations onto the gradient direction. However, since the original method was designed for the vector dimension, that is, if we generate a random initial perturbation $\bm{z} \sim \mathcal{N}(0,I_d)$, 
we hope that the perturbation could satisfy the condition that:
\begin{equation}
(\nabla \mathcal{L}^T \bm{z})^2 = \delta \|\nabla \mathcal{L}\|^2
\label{eq4}
\end{equation}
%
which can be simplified to:
\begin{equation}
\langle\nabla \mathcal{L}, \bm{z}\rangle = \xi \cdot \sqrt{\delta} \|\nabla \mathcal{L}\|
\label{eq5}
\end{equation}
where $\xi$ is a constant that is randomly selected from the set $\{ -1,1\}$. And $\langle \cdot, \cdot \rangle$ indicates the inner product of two vectors. However, directly generating the perturbation vector corresponding to Equation \ref{eq4} and \ref{eq5} is difficult since it requires sampling from a constrained space rather than the free full parameter space. 
Since Equation \ref{eq5} corresponds to a hyperplane in the vector space,
we can randomly sample an initial perturbed vector $\bm{v}_{init}$ which can be decomposed into two components: one parallel with the hyperplane and the other orthogonal to it. We can denote them as $\bm{v}_{init\parallel}$ and $\bm{v}_{init\perp}$, respectively.
Then, we only need to retain the parallel component $\bm{v}=\bm{v}_{init\parallel}$, which satisfies the requirement of Equation \ref{eq5}. According to \citep{ma2025revisiting}, we can 
calculate the parallel component 
$\bm{v}$ 
of the initial perturbation $\bm{v}_{init}$ as follows:
\begin{equation}
\small
\bm{v} = \bm{v}_{init} - \frac{\nabla \mathcal{L}^{T} \bm{v}_{init} - \xi \sqrt{\delta}\,\lVert \nabla \mathcal{L} \rVert}{\lVert \nabla \mathcal{L} \rVert^{2}}\, \nabla \mathcal{L}
\label{eq6}
\end{equation}
In Equation \ref{eq6} the aligned perturbation $\bm{v}$ is not only consistent with the gradient direction and but also satisfies the Gaussian distribution condition, satisfying the following requirements for the chosen perturbations to reduce the variance of ZO gradient estimation~\citep{ma2025revisiting,article,gao2022generalizinggaussiansmoothingrandom}:
\begin{itemize}
    \item \textbf{(a) Constant Magnitude}: The magnitude ($\ell_{2}$ norm) of the perturbation vector $\bm{v}$ is a fixed constant, i.e., $||\bm{v}||^2 = d\delta$ ($\delta$ is a random constant). Many traditional methods fall into this category, such as Gaussian distribution, Rademacher distribution and uniform distribution.
    \item \textbf{(b) Directional Alignment}: The square of the inner product between the perturbation vector $\bm{v}$ and the true gradient $\nabla \mathcal{L}$ is a fixed value, i.e., $(\nabla \mathcal{L}^T \bm{v})^2 = \delta \|\nabla \mathcal{L}\|^2$. This condition implies that the perturbation direction should be 'aligned' with the gradient direction.
\end{itemize}

We now extend this theory to the case of high-dimensional matrices. 
The vector norm on the right-hand side of Equation~\ref{eq5} can be naturally generalized to the matrix norm. 
In this paper, we adopt the Frobenius norm for matrices, i.e. $\|\bm{A}\|_{F} = \sqrt{\sum_{i,j} a_{ij}^{2}}$, where $a_{ij}$ is the number in the $i$-th row and $j$-th column of the matrix $\bm{A}$. We can replace the vector inner product with the Frobenius inner product for matrices without loss of generality.
For two matrices $\bm{A},\bm{B} \in \mathbb{R}^{m\times n}$, we define
\begin{equation}
\small
\langle \bm{A}, \bm{B}\rangle_F \;=\; \mathrm{Tr}(\bm{A}^\top \bm{B})
\
\label{eq7}
\end{equation}
where $\mathrm{Tr}(\cdot)$ in Equation \ref{eq7} means the trace of a matrix and $b_{ij}$ is the number in the $i$-th row and $j$-th column of the matrix $\bm{B}$.
Therefore, the vector hyperplane in Equation \ref{eq5} can be extended to a tensor hyperplane:
\begin{equation}
\langle\nabla \mathcal{\bm{L}}, \bm{Z}\rangle_F = \xi \cdot \sqrt{\delta} \,\|\nabla \mathcal{\bm{L}}\|_F
\label{eq8}
\end{equation}
where $\bm{Z}$ is a random perturbation satisfying Gaussian distribution.

Similarly, if we randomly generate an initial perturbation matrix 
$\bm{C}_{init} \sim \mathcal{N}(0,I_{m \times n})$ and $\bm{C}_{init}$ have equivalent dimension with gradient matrix $\nabla \mathcal{\bm{L}} \in \mathbb{R}^{m\times n}$, then the sampled initial perturbation can also be decomposed into a parallel component ($\bm{C}_{init\parallel}$) and a vertical component ($\bm{C}_{init\perp}$). Deriving from Equation~\ref{eq6}, the parallel component $\bm{C}=\bm{C}_{init\parallel}$ of $\bm{C}_{init}$
can be formulated
as:
%
\begin{equation}
\bm{C} = \bm{C}_{init} - 
\frac{\langle \nabla \mathcal{\bm{L}}, \bm{C}_{init}\rangle_F - \xi \sqrt{\delta}\,\lVert \nabla \mathcal{L} \rVert_F}
{\lVert \nabla \mathcal{\bm{\bm{L}}} \rVert_F^{2}} \,\nabla \mathcal{\bm{L}}
\label{eq9}
\end{equation}
We only need to retain the parallel component $\bm{C}$ of the hyperplane in Equation \ref{eq8}, i.e., the one aligned with the gradient direction and the subsequent ZO perturbation update is then performed using only the parallel component.

\subsection{Low-dimensional Gradient Space Design}

\textbf{Motivation.} 
If we directly apply Equation \ref{eq9} for perturbation alignment, there are two issues:
First, for large language models such as OPT-6.7B, the Transformer layer matrices are very large (e.g., 4096×4096), which leads to high computational cost.
Second, Equation \ref{eq9} still performs perturbation alignment in the full parameter space. However, as we have shown in the Appendix \ref{vr}, the larger the dimensionality of the perturbations, the higher the variance of the ZO gradient estimation.
This motivates us to explore whether it is possible to restrict the perturbations to a low-dimensional space and perform the perturbation alignment with gradient direction within this low-dimensional space.

Suppose the gradient matrix is denoted as $\bm{S}=\nabla\mathcal{L} \in \mathbb{R}^{m\times n}$, 
it can be decomposed in the format of the product of an orthogonal basis matrix and a coefficient matrix,
using techniques such as singular value decomposition (SVD) or QR decomposition.
In this work, we adopt SVD for low-rank decomposition, then we have:
%
\begin{equation}
\bm{S} \simeq \bm{U}_r \bm{S}_r \bm{V}_r^T
\label{eq10}
\end{equation}

where $\bm{U}_r \in \mathbb{R}^{m\times r}, \bm{S}_r \in \mathbb{R}^{r\times r}, \bm{V}_r \in \mathbb{R}^{n\times r}$, $r\ll m$ and $r\ll n$. Evidently, $\bm{U}_r$ and $\bm{V}_r$ can be regarded as a pair of frames, i.e., two orthogonal bases. And $\bm{S}_r$ serves as the set of scaling factors associated with the bases, which indicates the importance of the direction of each singular vector. Hence, a natural choice is to preserve the leading $r$ directions, which captures the most significant components. 
Then, by combining Equation \ref{eq9} and Equation \ref{eq10}, we have
\begin{equation}
\small
\bm{C} = \bm{C}_{init} - 
\frac{\langle \bm{U}_r \bm{S}_r \bm{V}_r^T, \bm{C}_{init}\rangle_F - \xi \sqrt{\delta}\,\lVert \bm{U}_r \bm{S}_r \bm{V}_r^T \rVert_F}
{\lVert \bm{U}_r \bm{S}_r \bm{V}_r^T \rVert_F^{2}} \,\bm{U}_r \bm{S}_r \bm{V}_r^T.
\label{eq11}
\end{equation}

\subsection{Adapting Projected Gradient-Aligned Perturbation in Low-dimensional Gradient Space}

So far, we can conduct gradient alignment with Equation \ref{eq11}. However, generating perturbation in full parameter space will lead to large variance of ZO gradient estimation. To further reduce the variance, we propose to generate perturbations from a lower dimension space, therefore reducing the number of perturbed parameters, resulting in reduced variance.
Since the Frobenius inner product has the feature of:
%
\begin{equation}
\langle \bm{U}_r \bm{S}_r \bm{V}_r^T, \bm{C}_{init}\rangle_F=\langle  \bm{S}_r , \bm{U}_r^T \bm{C}_{init} \bm{V}_r\rangle_F
\label{eq12}
\end{equation}

Evidently, $\bm{C}_{init}\in \mathbb{R}^{m\times n}$ has been transformed into a lower dimension perturbation $\bm{U}_r^T \bm{C}_{init} \bm{V}_r \in \mathbb{R}^{r\times r}$. 
For simplicity, we denote $\mathcal{Z}_{init}=\bm{U}_r^T \bm{C}_{init} \bm{V}_r\in \mathbb{R}^{r\times r} $.
Based on the property of
$ \lVert \bm{U}_r \bm{S}_r \bm{V}_r^T \rVert_F = \lVert \bm{S}_r  \rVert_F$, we can simplify Equation \ref{eq11} to:
\begin{equation}
\bm{C} = \bm{C}_{init} - 
\frac{\langle \bm{S}_r , \mathcal{Z}_{init}\rangle_F - \xi \sqrt{\delta}\,\lVert  \bm{S}_r  \rVert_F}
{\lVert  \bm{S}_r  \rVert_F^{2}} \,\bm{U}_r \bm{S}_r \bm{V}_r^T
\label{eq13}
\end{equation}

Since $\bm{U}_r^T\bm{U}_r=\bm{I}_m$ and $\bm{V}_r^T\bm{V}_r=\bm{I}_n$, we perform left multiplication with $\bm{U}_r^T$ on both sides of Equation~(13), and right multiplication with $\bm{V}_r$ on both sides as well. Then we have:
\begin{equation}
\bm{U}_r^T\bm{C}\bm{V}_r = \mathcal{Z}_{init} - 
\frac{\langle \bm{S}_r , \mathcal{Z}_{init}\rangle_F - \xi \sqrt{\delta}\,\lVert  \bm{S}_r  \rVert_F}
{\lVert  \bm{S}_r  \rVert_F^{2}} \,\bm{S}_r .
\label{eq14}
\end{equation}
Similarly, we can use $\mathcal{Z} \in \mathbb{R}^{r\times r}$ to denote $\bm{U}_r^T\bm{C}\bm{V}_r$.
Then, the hyperplane condition in Equation \ref{eq8} can be satisfied by the projected perturbation $\mathcal{Z}$:
\begin{equation}
\langle \bm{S}^r_\ell,\mathcal{Z}\rangle_F=\xi\sqrt{\delta}\|\bm{S}^r_\ell\|_F
\label{eq15}
\end{equation}

So far,
from the derivation,
we can obtain the final component in the low-dimensional space that is parallel to the hyperlane defined by the low-dimensional gradient coefficient matrix,
only need to generate an initial Gaussian perturbation $\mathcal{Z}_{init} \sim \mathcal{N}(0,I_{r \times r})$ from a lower-dimensional space and refine it through projection from Equation \ref{eq14} to get corrected low-dimension perturbation $\mathcal{Z}$.
Finally, 
we multiply the matrix $\mathcal{Z}$ with the frame matrix $\bm{U}_r, \bm{V}_r$ to obtain the representation of the low-dimensional perturbation in the high-dimensional space $\mathcal{Z}_f=\bm{U}_r \mathcal{Z}\bm{V}_r^T$.

In P-GAP, since the true gradient direction of the loss surface is unknown at each step in the ZO fine-tuning setting, we adopt a lazy update strategy that has been shown effective in prior works \citep{rando2024stochastic, liu2018zeroth, yu2024subzero}. The overall procedure of P-GAP is summarized in \textbf{Algorithm~1} in Appendix \ref{al}. Specifically, we first choose the update interval $k$ (window size), the number of probe perturbations $h$, the number of basis columns $r$, the projection magnitude $\delta$, and other hyperparameters. Every $k$ steps, we use lazy update strategy to estimate an approximate gradient direction using $h$ random probe perturbations and update the basis matrices $\bm{U}_r,\bm{V}_r$ and coefficient matrix $\bm{S}_r$ for each parameter $\bm{W}$. During the following $k$ iterations, we reuse the same basis and coefficient matrices to construct low-dimensional perturbation representations $\mathcal{Z}$, which are mapped back to the original parameter space to get $\mathcal{Z}_f$ for ZO updates. 


\section{Experiments}
\vspace{-0.1in}



\textbf{Datasets.} 
We evaluate P-GAP with both classification datasets such as SST-2, SST-5, RTE and generation tasks such as SQuAD, DROP. For RoBERTa-large, we follow prior ZO studies \citep{malladi2023fine, zhao2024second, yu2024subzero} and use $k\!=\!16$ as few-shot examples and $k\!=\!512$ as many-shot examples per class, evaluated on 1,000 test samples, for classification tasks. For autoregressive models, we use fixed splits of 1000, 500, 1000 for train, evaluation, test, respectively, and include both classification (e.g., SST2) and generation tasks (e.g., SQuAD) to assess generalization.
\vspace{-0.05in}

\textbf{Models and Baselines.}
Our experiments span both masked and autoregressive large language models. For the masked model, we use RoBERTa-large (350M)~\citep{liu2019roberta} following MeZO~\citep{malladi2023fine}, while for autoregressive modeling we include representative families such as OPT~\citep{zhang2022opt} and LLaMA~\citep{touvron2023llama}, covering model sizes from hundreds of millions to several billions of parameters (e.g., RoBERTa-large, OPT-2.7B/6.7B, and LLaMA-3-3B/8B). 
We compare P-GAP with representative state-of-the-art zeroth-order optimization baselines, including MeZO~\citep{malladi2023fine},
HiZOO~\citep{zhao2024second},
SubZero~\citep{yu2024subzero},
and Sparse-MeZO~\citep{liu2024sparse}.
For SubZero and Sparse-MeZO on OPT-13B, we adopt the results reported in \cite{yu2024subzero} due to the lack of open-sourced implementations.

\vspace{-0.05in}

\textbf{Implementation Details and Hyperparameter Settings.}
All experiments are conducted on NVIDIA A100 GPUs. To ensure a fair comparison, for key hyperparameters such as the batch size, and optimization schedule, we use the same setting as MeZO \citep{malladi2023fine}. 
Our detailed hyperparameter settings such as $k$ and $\delta$ can be found in Appendix \ref{al}.

\vspace{-0.1in}


\subsection{Results on medium-sized model}
\vspace{-0.1in}

\begin{table}[htbp]
\centering
\small
\renewcommand{\arraystretch}{1.1}
\captionsetup{justification=raggedright,singlelinecheck=false}
\caption{
Experiments on RoBERTa-large 350M across different classification datasets and $k$ settings
}
\vspace{-0.1in}
\setlength{\tabcolsep}{4pt}
\resizebox{0.95\linewidth}{!}{
\begin{tabular}{llcccccc}
\toprule
\textbf{Task Type} & \textbf{Dataset} & \textbf{SST-2} & \textbf{SST-5} & \textbf{SNLI} & \textbf{MNLI} & \textbf{RTE} & \textbf{TREC} \\
\midrule
\multicolumn{2}{l}{Zero-shot} & 79.0 & 35.5 & 50.2 & 48.8 & 51.4 & 32.0 \\
\cmidrule(lr){1-8}
\multicolumn{8}{c}{\textbf{Gradient-free methods: $k=16$}} \\
\cmidrule(lr){1-8}
MeZO         &        & 90.5 (1.2) & 45.5 (2.0) & 66.0 (2.7) & 56.5 (2.5) & 59.4 (5.3) & 76.9 (2.7) \\
MeZO LoRA    &        & 85.8 (0.7) & 41.6 (0.8) & 64.9 (0.8) & 59.5 (1.5) & 61.7 (3.2) & 58.2 (5.6) \\
\rowcolor{cyan!7}   
P-GAP         &        & \textbf{91.4 (0.4)} & \textbf{47.3 (2.8)} & \textbf{70.4 (1.1)} & \textbf{63.3 (2.1)} & \textbf{65.7 (2.8)} & \textbf{82.8 (3.7)} \\
\rowcolor{cyan!7} 
P-GAP LoRA   &        & 86.3 (0.6) & 41.7 (1.5) & 65.2 (0.5) & 60.8 (1.9) & 61.7 (3.0) & 59.4 (2.1) \\
\cmidrule(lr){1-8}
\multicolumn{8}{c}{\textbf{Gradient-based methods: $k=16$}} \\
\cmidrule(lr){1-8}
FT           &        & 91.9 (1.8) & 47.5 (1.9) & 77.5 (2.6) & 70.2 (2.3) & 66.4 (7.2) & 85.0 (2.5) \\
FT LoRA      &        & 91.4 (1.7) & 46.7 (1.1) & 74.9 (4.3) & 67.7 (1.4) & 66.1 (3.5) & 86.1 (3.3) \\
\cmidrule(lr){1-8}
\multicolumn{8}{c}{\textbf{Gradient-free methods: $k=512$}} \\
\cmidrule(lr){1-8}
MeZO         &        & 93.3 (0.7) & 52.4 (1.2) & 83.0 (1.0) & 78.3 (0.5) & \textbf{78.6 (2.0)} & 94.3 (1.3) \\
MeZO LoRA    &        & 91.6 (0.8) & 44.8 (0.4) & 73.3 (0.6) & 66.4 (0.4) & 73.3 (1.5) & 63.8 (2.3) \\
\rowcolor{cyan!7}  
P-GAP         &        & \textbf{95.1 (0.6)} & \textbf{53.3 (1.7)} & \textbf{83.9 (2.3)} & \textbf{78.6 (0.9)} & 76.6 (1.2) & \textbf{94.8 (1.0)} \\
\rowcolor{cyan!7}  
P-GAP LoRA   &        & 92.9 (0.3) & 45.5 (0.6) & 74.1 (1.9) & 63.7 (1.2) & 74.0 (0.9) & 62.4 (2.8) \\
\cmidrule(lr){1-8}
\multicolumn{8}{c}{\textbf{Gradient-based methods: $k=512$}} \\
\cmidrule(lr){1-8}
FT           &        & 93.9 (0.7) & 55.9 (0.9) & 88.7 (0.8) & 84.4 (0.8) & 82.7 (1.4) & 97.3 (0.2) \\
FT LoRA      &        & 94.2 (0.2) & 55.7 (0.8) & 88.3 (0.5) & 86.9 (0.6) & 83.2 (1.3) & 97.0 (0.3) \\
\bottomrule
\end{tabular}}
\vspace{-0.1in}
\label{tab:robert-results}
\end{table}

\begin{table}[htbp]
\centering
\vspace{-0.0in}
\setlength{\tabcolsep}{4pt}
\renewcommand{\arraystretch}{1.1}
\captionsetup{justification=raggedright,singlelinecheck=false}
\caption{Results of fine-tuning OPT-2.7B on eight classification datasets and two generation datasets}
\resizebox{.95\linewidth}{!}{
\begin{tabular}{lccccccccccc}
\toprule
\textbf{Dataset} & \textbf{SST-2} & \textbf{RTE} & \textbf{CB} & \textbf{BoolQ} & \textbf{WSC} & \textbf{WIC} & \textbf{COPA} & \textbf{MultiRC} & \textbf{SQuAD} & \textbf{DROP} \\
\textbf{Task Type} & \multicolumn{8}{c}{\emph{classification}} & \multicolumn{2}{c}{\emph{generation}} \\
\midrule
Zero-shot   & 56.3 & 54.2 & 50.0 & 47.6 & 36.5 & 52.7 & 72.0   & 44.4 & 29.8 & 10.0 \\
FT          & 94.2 & 81.2 & 82.1 & 72.2 & 63.8 & 65.8 & 82.0   & 71.6 & 78.4 & 30.3 \\
LoRA        & 94.6 & 80.8 & 82.7 & 77.7 & 59.8 & 64.0 & 80.0   & 72.8 & 77.9 & 31.1 \\
\midrule
MeZO        & 91.2 & 63.5 & 71.4 & 67.4 & 62.5 & 59.2 & 76.0   & 59.4 & 66.8 & 19.4 \\
HiZOO       & 90.8 & 60.6 & 70.4 & \textbf{68.0} & 60.2 & 56.6 & 79.0   & 55.8 & 68.2 & 20.2 \\
\rowcolor{cyan!7} 
P-GAP        & \textbf{91.6} & \textbf{63.8} & \textbf{73.2} & 66.8 & \textbf{66.1} & \textbf{61.0} & \textbf{82.0} & \textbf{60.8} & \textbf{74.9} & \textbf{21.1} \\
\midrule
MeZO LoRA   & 91.0 & 62.8 & 67.8 & 64.8 & 65.4 & 58.2 & 79.0   & 63.4 & 63.4 & 19.2 \\
HiZOO LoRA  & 90.6 & \textbf{66.3} & \textbf{71.4} & 67.0 & 62.2 & 58.8 & 78.0   & 59.0 & 69.2 & 18.3 \\
\rowcolor{cyan!7} 
P-GAP LoRA   & \textbf{91.8} & 63.8 & \textbf{71.4} & \textbf{67.4} & \textbf{66.3} & \textbf{59.8} & \textbf{80.0} & \textbf{63.8} & \textbf{76.6} & \textbf{22.5} \\
\bottomrule
\end{tabular}}
\vspace{-0.1in}
\label{tab:opt-2.7}
\end{table}

We  conduct 
experiments 
on
classification datasets
to evaluate the effectiveness of P-GAP on RoBERTa-large 350M \citep{liu2019roberta} as shown in Table \ref{tab:robert-results}.
We observe that P-GAP can generally yield higher accuracy across multiple datasets. For instance, when $k=16$, P-GAP can achieve around
0.9\%, 6.8\%, and 6.3\% higher accuracy than MeZO on
SST-2, RTE and MNLI, respectively. 
\begin{table}[htbp]
\centering
\setlength{\tabcolsep}{2pt}
\renewcommand{\arraystretch}{0.9}

\begin{minipage}{0.48\textwidth}
\captionsetup{justification=raggedright, singlelinecheck=false}
\caption{Experiments on OPT-6.7B (with 1000 training samples)}
\vspace{-0.1in}
\label{tab:opt-6.7}
\resizebox{\linewidth}{!}{
\begin{tabular}{lccccc}
\toprule
\textbf{Dataset} & \textbf{SST-2} & \textbf{RTE} & \textbf{CB}  & \textbf{WSC}  & \textbf{SQuAD}  \\
\textbf{Task Type} & \multicolumn{4}{c}{\emph{classification}} & \multicolumn{1}{c}{\emph{generation}} \\
\midrule
MeZO      & 91.8 & 62.8 & 73.2 & 65.4 &  70.3 \\
HiZOO     & 90.9 & \textbf{66.3} & 71.4 & 62.1 & 71.9 \\
\rowcolor{cyan!7} P-GAP & \textbf{92.0} & 63.8 & \textbf{78.6} & \textbf{67.3} & \textbf{75.4} \\
\midrule
MeZO LoRA & 93.4 & 67.9 & 73.2 & 65.4 & 69.8 \\
HiZOO LoRA & 92.5 & 68.7 & 71.4 & 63.6 & 72.3 \\
\rowcolor{cyan!7} P-GAP LoRA & \textbf{94.0} & \textbf{72.5} & \textbf{78.6} & \textbf{66.3} & \textbf{79.2} \\
\bottomrule
\end{tabular}}
\end{minipage}
\hfill
\begin{minipage}{0.48\textwidth}
\captionsetup{justification=raggedright, singlelinecheck=false}
\caption{Experiments on OPT-13B (with 1000 training samples)}
\label{tab:opt-13}
\resizebox{\linewidth}{!}{
\begin{tabular}{lcccc}
\toprule
\textbf{Dataset} & \textbf{SST-2} & \textbf{RTE} & \textbf{WSC} & \textbf{SQuAD}  \\
\textbf{Task Type} & \multicolumn{3}{c}{\emph{classification}} & \multicolumn{1}{c}{\emph{generation}} \\
\midrule
MeZO        & 91.4 & 69.3 & 61.5 & 84.2 \\
HiZOO       & 92.1 & 66.1 & 63.5 & 81.9 \\
Sparse-MeZO & 92.3 & \textbf{76.9} & 61.1 & 77.9 \\
Subzero     & 92.1 & 74.0 & 65.4 & 84.5 \\
\rowcolor{cyan!7} P-GAP & \textbf{92.7} & 73.8 & \textbf{66.3} & \textbf{85.0} \\
\bottomrule
\end{tabular}}
\end{minipage}
\end{table}
%
%
%
To further investigate its flexibility, we evaluate P-GAP within the PEFT framework, LoRA framework.
We observe that LoRA typically incurs a modest degradation in performance compared to full-model FT, P-GAP remains highly competitive: it can generally outperform zeroth-order baselines and maintains good performance even when the number of trainable parameters is significantly reduced. 
These results can show that our approach is effective in both full-tuning regime and
PEFT scenarios such as LoRA, highlighting its robustness and practicality for medium-sized language model deployment.

\vspace{-0.05in}

\subsection{Results on large autoregressive Models}


\begin{wrapfigure}[11]{r}{0.5\textwidth}
\vspace{-0.2in}
\centering
\includegraphics[width=0.99\linewidth]{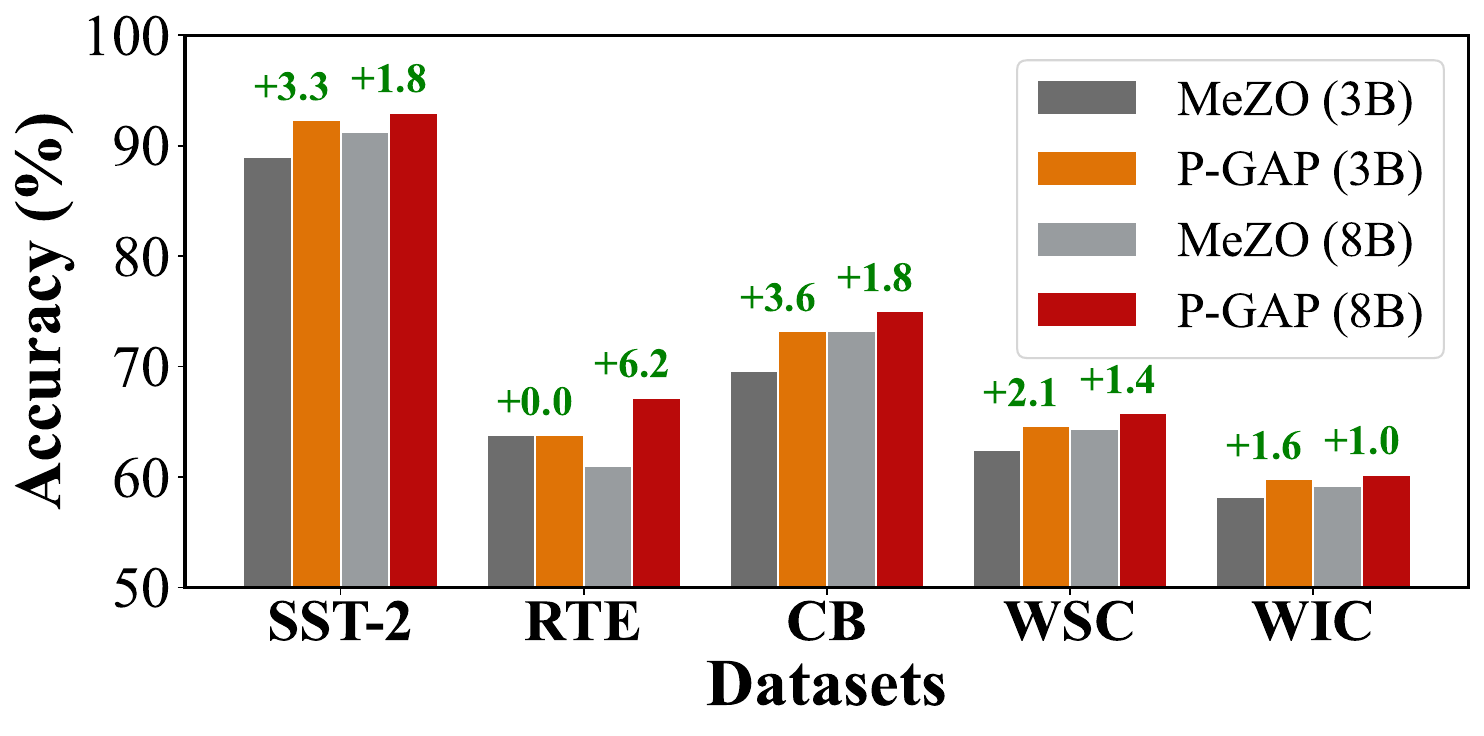}
\vspace{-0.3in}
\captionsetup{justification=raggedright,singlelinecheck=false}
\caption{Accuracy comparison of MeZO and P-GAP (Ours) on LLaMA3-3B and LLaMA3-8B}
\label{fig:llama}
\vspace{-0.3in}
\end{wrapfigure}

P-GAP is evaluated with both the OPT and LLaMA model families, on classification tasks such as RTE and SST-2 datasets, and generation tasks such as SQuAD and DROP datasets. 
As shown in Table \ref{tab:opt-2.7},
on OPT-2.7B, P-GAP consistently 
outperforms MeZO and HiZOO.
For instance, on COPA, P-GAP can achieve an accuracy of 82.0\%, which is 6\% higher than
MeZO at 76.0\% and also surpasses HiZOO 
with an increase about 3\%. On generation tasks, P-GAP can obtain 74.9\% accuracy on SQuAD, yielding a 12\% increase compared to
MeZO (66.8\%).
When combined with LoRA, our approach remains competitive and continues to outperform baselines. On SQuAD with LoRA, P-GAP reaches about 76.6\% accuracy, exceeding MeZO LoRA (63.4\%)
by more than 13\%. 


Turning to LLaMA-3 models, Figure~\ref{fig:llama} shows that P-GAP can generally boost accuracy across datasets. For example, on SST-2 datasets, P-GAP can acheive about 3.3\% increase in accuracy on LLaMA-3-3B and 1.8\% increase of accuracy on LLaMA-3-8B over MeZO baseline.

\vspace{-0.15in}

\subsection{Performance on LLMs with Various Scales}
\vspace{-0.05in}


We also evaluate the performance of P-GAP on LLMs with different scales. For example, we conduct experiments 
on OPT-6.7B, OPT-13B as shown in Table \ref{tab:opt-6.7}, Table \ref{tab:opt-13}, respectively. We evaluate P-GAP with 
LLaMA-3-3B and LLaMA-3-8B as shown in
Figure \ref{fig:llama}.
We observe that P-GAP has consistent advantages over baselines on OPT-6.7B, OPT-13B and LLaMA-3 models.
On OPT-6.7B with the CB dataset, P-GAP achieves 78.6\% accuracy, outperforming MeZO by 5.4\% and HiZOO by 7.2\%, individually. On SQuAD, it can achieve an accuracy of 75.4\%, which is about 5.1\% higher than MeZO. When combined with LoRA, the improvements of P-GAP become even more significant: P-GAP reaches 72.5\% accuracy on RTE and 79.2\% on SQuAD, surpassing HiZOO by nearly 4\% and 7\%, respectively. For OPT-13B model, P-GAP can achieve about 66.3\% accuracy in fine-tuning WSC dataset, surpassing all of the baselines including Sparse-MeZO and Subzero.
%
%
%


\begin{figure}[ht]
    \centering
    \begin{subfigure}[t]{0.32\textwidth}
        \centering
        \includegraphics[width=\linewidth]{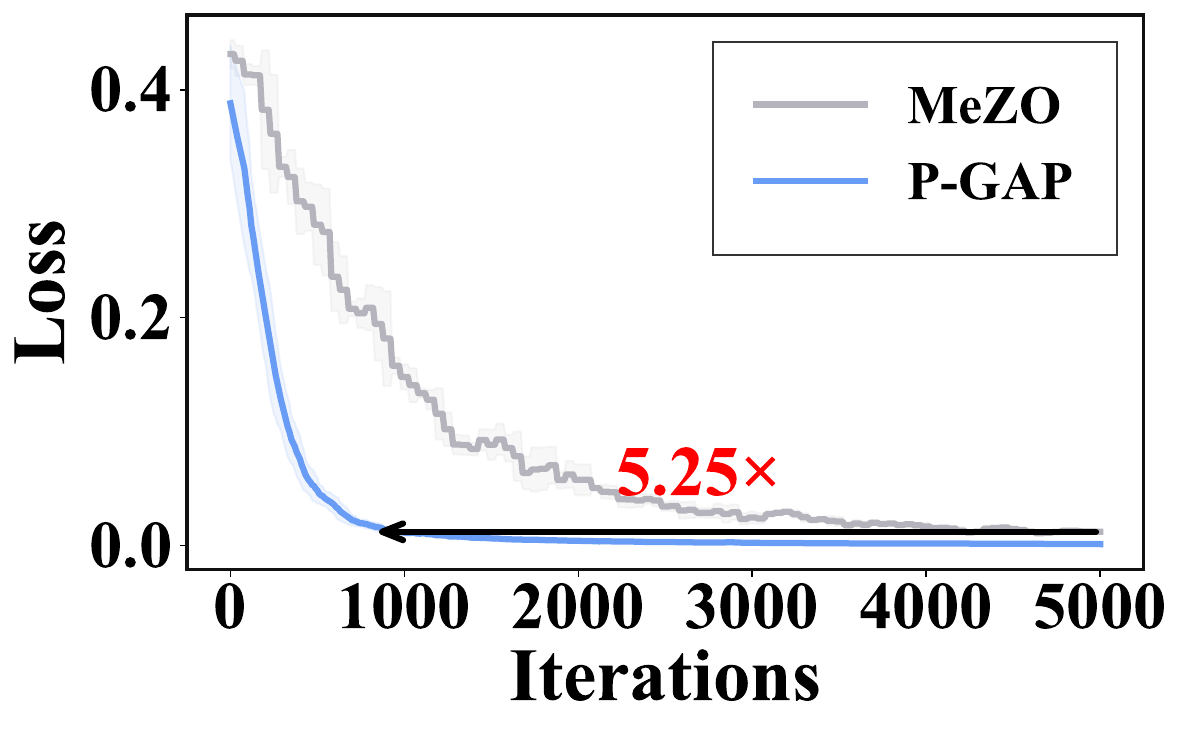}
        \caption{SST-2 Dataset}
    \end{subfigure}
    \hfill
    \begin{subfigure}[t]{0.32\textwidth}
        \centering
        \includegraphics[width=\linewidth]{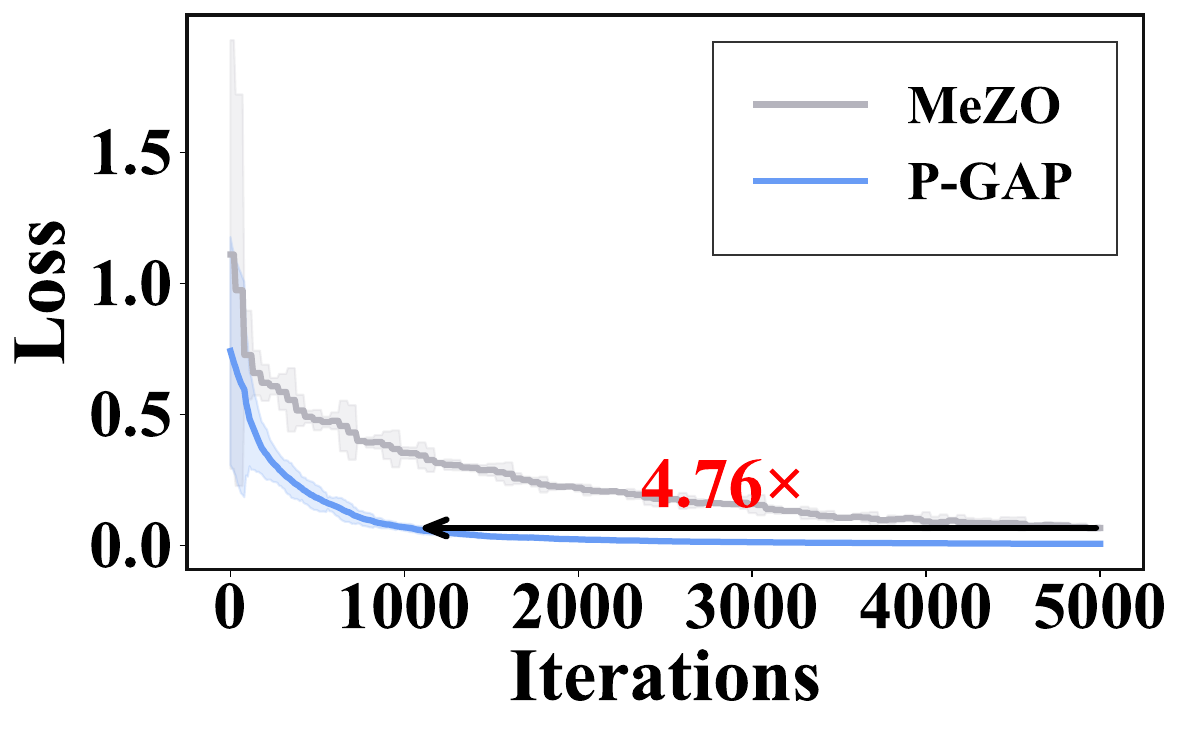}
        \caption{RTE Dataset}
    \end{subfigure}
    \hfill
    \begin{subfigure}[t]{0.32\textwidth}
        \centering
        \includegraphics[width=\linewidth]{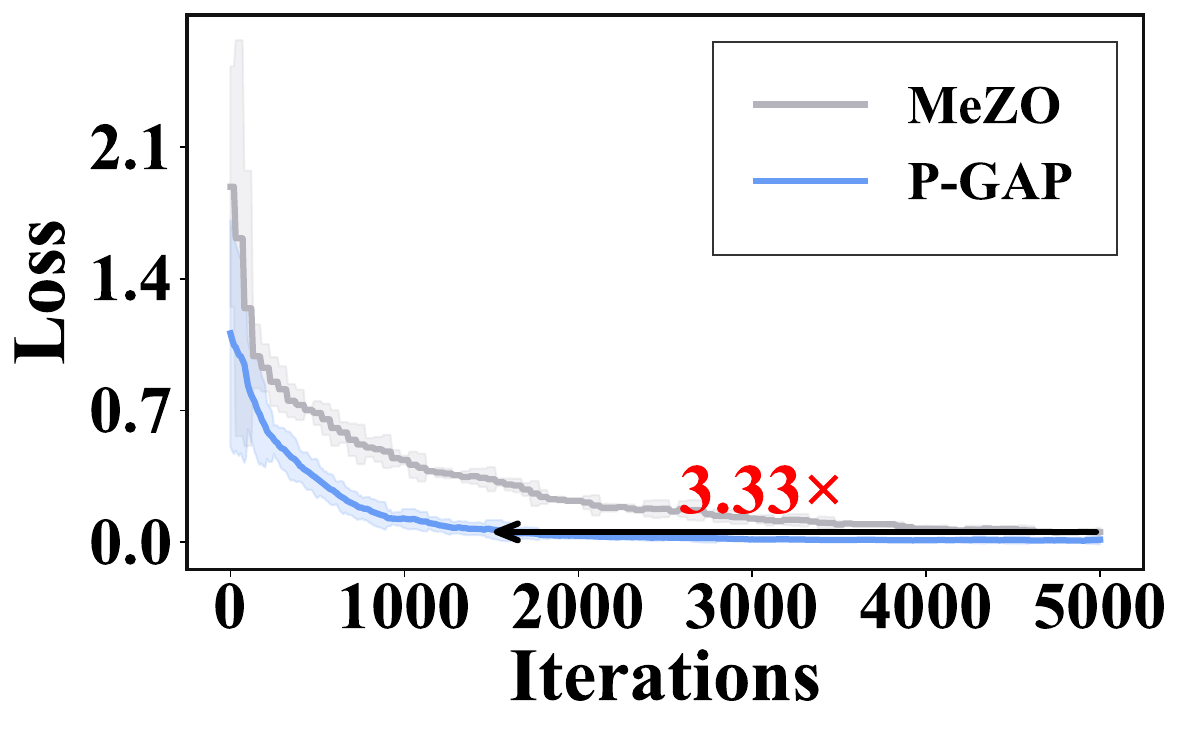}
        \caption{SNLI Dataset}
    \end{subfigure}
    \vspace{-0.05in}
    \caption{Training loss comparison with iterations of MeZO and P-GAP on RoBERTa-large}
    \vspace{-0.2in}
    \label{fig:rob_loss}
\end{figure}

\begin{figure}[H]
    \centering
    \includegraphics[width=0.99\textwidth]{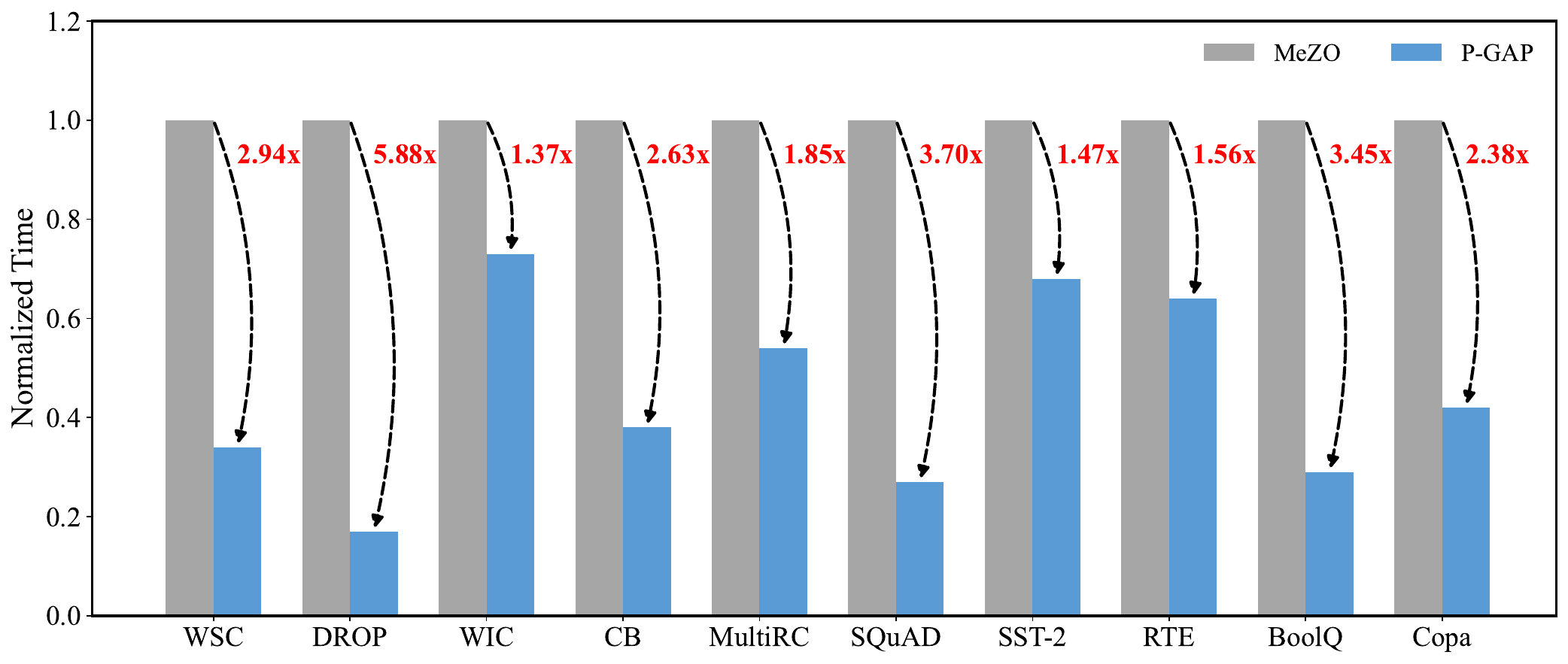}
    \captionsetup{justification=raggedright, singlelinecheck=false}
    \caption{Comparison of GPU hours for full FT across different datasets on OPT-2.7B between MeZO and P-GAP. Results are presented as normalized time (numbers in red indicate speedup)}
    \vspace{-0.2in}
    \label{fig:timecompare}
\end{figure}

\begin{wrapfigure}[12]{r}{0.48\textwidth}
\vspace{-0.15in}
\centering
\includegraphics[width=0.92\linewidth]{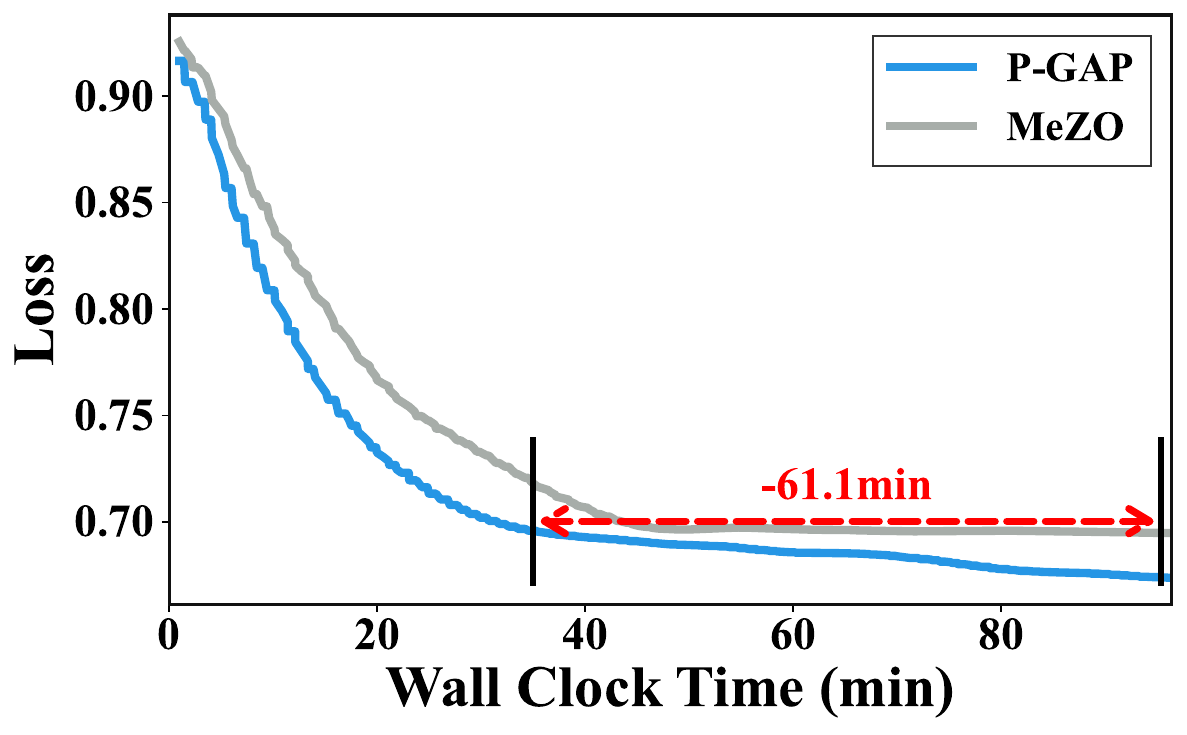}
\vspace{-0.15in}
\captionsetup{justification=raggedright,singlelinecheck=false}
\caption{Wall clock time  for OPT-2.7B on CB datasets}
\label{fig:wall}
\end{wrapfigure}

\subsection{Convergence and Wall-clock Time Analysis}




We provide the convergence and wall-clock time analysis on different models to show the acceleration effects of P-GAP over baseline.
As shown in Figure \ref{fig:rob_loss}, on RoBERTa-large, our approach achieves lower training loss more quickly, reducing the number of iterations by 5.25$\times$ on SST-2 and 3.33$\times$ on SNLI for achiving the same final loss as MeZO. This demonstrates that fewer update steps are sufficient for P-GAP to achieve competitive performance.
Figure \ref{fig:timecompare} 
shows the overall normalized GPU hours of P-GAP compared to MeZO for fine-tuning on all of the ten datasets. We can observe that P-GAP consistently accelerates convergence compared to MeZO  \citep{malladi2023fine} across datasets.
For example,
on DROP, 
P-GAP can achieve about 5.88$\times$ speedup (with only 17\% of the training time) compared to MeZO, while also achieving better performance. 
P-GAP can also reduce wall-clock time. With OPT-2.7B on the CB dataset, P-GAP reach the loss of 0.6985 about 61.1 minutes earlier than MeZO,
corresponding to a reduction of 40\% in convergence time, as shown in Figure \ref{fig:wall}.
These results highlight the high efficiency of P-GAP, which not only reduces the number iterations but also achieves practical time savings during training.

\subsection{Memory analysis}

We evaluate P-GAP’s memory usage and training efficiency under both full-parameter and LoRA-based fine-tuning. As shown in Table~\ref{tab:mem}, our approach strikes a favorable balance between convergence speed and per-step overhead. Compared to MeZO, which requires the full training budget of 100\% iterations and GPU hours, P-GAP reduces the number of iterations to only 15.6\% and the total GPU hours to 27.3\%, with memory usage slightly larger than MeZO and smaller than HiZOO. On SQuAD dataset, this translates to more than a 70\% reduction in training time with comparable accuracy.
\begin{wraptable}[14]{r}{0.5\textwidth}
\centering
\renewcommand{\arraystretch}{1.1}
\captionsetup{justification=raggedright,singlelinecheck=false}
\caption{Memory and training time comparison on OPT-2.7B (SQuAD, avg. 300 tokens)}
\vspace{-0.1in}
\label{tab:gpu_hour}

\resizebox{0.5\textwidth}{!}{
\begin{tabular}{lccc}
\toprule
\textbf{Method} & \textbf{Mem.} & \textbf{Iter.} & \textbf{Hours} \\
\midrule
FT         & 73.5G & 9.3\%   & 16.8\% \\
LoRA       & 58.5G & 6.3\%   & 11.5\% \\
\midrule
MeZO       & 9.4G & 100.0\% & 100.0\% \\
HiZOO      & 13.3G & 66.7\%  & 91.5\% \\
P-GAP     & 11.3G & 15.6\%  & 27.3\% \\
\midrule
MeZO+LoRA  & 8.4G  & 94.2\%  & 51.6\% \\
HiZOO+LoRA & 11.6G  & 80.0\%  & 65.7\% \\
P-GAP+LoRA& 9.1G  & 12.5\%  & 22.4\% \\
\bottomrule
\end{tabular}
}
\label{tab:mem}
\vspace{-1.0em}
\end{wraptable}
From a memory standpoint, P-GAP is substantially more efficient than gradient-based fine-tuning approaches such as full fine-tuning with 73.5G memory usage and LoRA with 58.5G memory usage, since it avoids storing gradients and activations. Even under parameter-efficient settings, it maintains strong efficiency. For instance, with LoRA, P-GAP further lowers GPU hours to 22.4\%, compared to 51.6\% for MeZO+LoRA and 65.7\% for HiZOO+LoRA, using only 9.1G of memory. These results highlight that P-GAP achieves faster convergence with minimal memory overhead across diverse tuning regimes. We provide more memory usage results in Appendix \ref{more}.

\section{Related Work}

\textbf{Memory-efficient Fine-tuning of LLMs. }
Large pre-trained models~\citep{radford2021learning,chen2022visualgpt,singh2022revisiting} have been increasingly employed across diverse domains. However, a tension arises between the growing demand for fine-tuning and the prohibitive computational cost, particularly in resource-constrained environments
~\citep{zeng2024flightllm, tan2025perturbation}. To mitigate this issue, several memory-efficient fine-tuning (PEFT) techniques have been proposed. For instance, 
~\cite{hu2021lora, dettmers2023qlora, liu2024dora, qin2024accurate} update only a subset of model parameters, while reducing memory usage. 
~\cite{frantar2022gptq, xiao2023smoothquant, dettmers2023qlora} compresses continuous real-valued weights into low-bit discrete formats (e.g., INT8 or INT4), thereby lowering both memory and computational costs.
Recently, zeroth-order (ZO) optimization has emerged as a promising paradigm for memory-efficient fine-tuning~\citep{malladi2023fine, zhang2024revisiting, chen2023deepzero, yu2024subzero}. By estimating gradients solely through forward passes, ZO eliminates the need to store memory-intensive activations and optimizer states~\citep{malladi2023fine, liu2024sparse, tang2024private}. 

\textbf{Acceleration of Zeroth-order Optimization. }
Despite the appealing memory-efficiency of ZO, the gains from ZO approaches come with a cost: convergence is often slower than FO alternatives, largely due to the inherent noise in randomized perturbation-based estimators. 
\cite{ji2019improved} proposed two new zeroth-order variance-reduced algorithms, ZO-SVRG-Coord-Rand and ZO-SPIDER-Coord, and provided refined theoretical analysis for the existing ZO-SVRG-Coord method in the context of nonconvex optimization, which can achieve better convergence rates and function query complexities than previous methods. 
\cite{duchi2015optimal} examined random perturbations with finite fourth-order moments, and demonstrated that using a uniform random vector yields the optimal dependence on the dimension $d$. \cite{kozak2023zeroth} construct orthogonal perturbations, or orthogonalize the sampled directions, so that the estimator can better explore diverse gradient directions and identify more effective descent paths. \citet{sener2020learning} propose to learn a latent low-dimensional manifold in the course of optimization, from which samples are drawn to effectively reduce sample complexity.

\section{Conclusion}
In this paper, we introduce \textbf{P-GAP}, a novel zeroth-order optimization framework for large language model fine-tuning  
by estimating a low-dimensional gradient space and aligns perturbations in projected gradients’ direction within the
space.
We provide theoretical analysis on how the variance of standard ZO estimators scales with the model size and how our approach can mitigate this problem through gradient estimation within low-dimension space.  
Extensive experiments show that P-GAP can effectively reduce the variance of ZO gradient estimation with improved accuracy and efficiency, and accelerated convergence.
%
For instance, P-GAP achieves up to 12\% increase in accuracy over baselines on SQuAD dataset, more than 61 minutes reduction in training time on BoolQ dataset. Overall, our findings highlight the potential of projected gradient-aligned perturbations for scalable and efficient ZO LLM fine-tuning in practice.

\bibliographystyle{iclr2026_conference}  
\bibliography{iclr2026_conference}       

\begin{thebibliography}{47}
\providecommand{\natexlab}[1]{#1}
\providecommand{\url}[1]{\texttt{#1}}
\expandafter\ifx\csname urlstyle\endcsname\relax
  \providecommand{\doi}[1]{doi: #1}\else
  \providecommand{\doi}{doi: \begingroup \urlstyle{rm}\Url}\fi

\bibitem[Castillo et~al.(2016)Castillo, Rafeiro, et~al.]{castillo2016introductory}
Ren{\'e}~Erl{\'\i}n Castillo, Humberto Rafeiro, et~al.
\newblock \emph{An introductory course in Lebesgue spaces}.
\newblock Springer, 2016.

\bibitem[Chen et~al.(2023)Chen, Zhang, Jia, Diffenderfer, Liu, Parasyris, Zhang, Zhang, Kailkhura, and Liu]{chen2023deepzero}
Aochuan Chen, Yimeng Zhang, Jinghan Jia, James Diffenderfer, Jiancheng Liu, Konstantinos Parasyris, Yihua Zhang, Zheng Zhang, Bhavya Kailkhura, and Sijia Liu.
\newblock Deepzero: Scaling up zeroth-order optimization for deep model training.
\newblock \emph{arXiv preprint arXiv:2310.02025}, 2023.

\bibitem[Chen et~al.(2022)Chen, Guo, Yi, Li, and Elhoseiny]{chen2022visualgpt}
Jun Chen, Han Guo, Kai Yi, Boyang Li, and Mohamed Elhoseiny.
\newblock Visualgpt: Data-efficient adaptation of pretrained language models for image captioning.
\newblock In \emph{Proceedings of the IEEE/CVF Conference on Computer Vision and Pattern Recognition}, pp.\  18030--18040, 2022.

\bibitem[Chen et~al.(2024)Chen, Zhang, Cao, Yuan, and Wen]{chen2024enhancing}
Yiming Chen, Yuan Zhang, Liyuan Cao, Kun Yuan, and Zaiwen Wen.
\newblock Enhancing zeroth-order fine-tuning for language models with low-rank structures.
\newblock \emph{arXiv preprint arXiv:2410.07698}, 2024.

\bibitem[Dettmers et~al.(2023)Dettmers, Pagnoni, Holtzman, and Zettlemoyer]{dettmers2023qlora}
Tim Dettmers, Artidoro Pagnoni, Ari Holtzman, and Luke Zettlemoyer.
\newblock Qlora: Efficient finetuning of quantized llms.
\newblock \emph{Advances in neural information processing systems}, 36:\penalty0 10088--10115, 2023.

\bibitem[Duchi et~al.(2015)Duchi, Jordan, Wainwright, and Wibisono]{duchi2015optimal}
John~C Duchi, Michael~I Jordan, Martin~J Wainwright, and Andre Wibisono.
\newblock Optimal rates for zero-order convex optimization: The power of two function evaluations.
\newblock \emph{IEEE Transactions on Information Theory}, 61\penalty0 (5):\penalty0 2788--2806, 2015.

\bibitem[Frantar et~al.(2022)Frantar, Ashkboos, Hoefler, and Alistarh]{frantar2022gptq}
Elias Frantar, Saleh Ashkboos, Torsten Hoefler, and Dan Alistarh.
\newblock Gptq: Accurate post-training quantization for generative pre-trained transformers.
\newblock \emph{arXiv preprint arXiv:2210.17323}, 2022.

\bibitem[Gao \& Sener(2022)Gao and Sener]{gao2022generalizinggaussiansmoothingrandom}
Katelyn Gao and Ozan Sener.
\newblock Generalizing gaussian smoothing for random search, 2022.
\newblock URL \url{https://arxiv.org/abs/2211.14721}.

\bibitem[Garling(2007)]{garling2007inequalities}
David~JH Garling.
\newblock \emph{Inequalities: a journey into linear analysis}.
\newblock Cambridge University Press, 2007.

\bibitem[Gautam et~al.(2024)Gautam, Park, Zhou, Raman, and Ha]{gautam2024variancereducedzerothordermethodsfinetuning}
Tanmay Gautam, Youngsuk Park, Hao Zhou, Parameswaran Raman, and Wooseok Ha.
\newblock Variance-reduced zeroth-order methods for fine-tuning language models, 2024.
\newblock URL \url{https://arxiv.org/abs/2404.08080}.

\bibitem[Gu et~al.(2021)Gu, Feng, Zhao, Ying, Chen, and Pan]{gu2021efficient}
Jiaqi Gu, Chenghao Feng, Zheng Zhao, Zhoufeng Ying, Ray~T Chen, and David~Z Pan.
\newblock Efficient on-chip learning for optical neural networks through power-aware sparse zeroth-order optimization.
\newblock In \emph{Proceedings of the AAAI conference on artificial intelligence}, volume~35, pp.\  7583--7591, 2021.

\bibitem[Hu et~al.(2021)Hu, Shen, Wallis, Allen-Zhu, Li, Wang, Wang, and Chen]{hu2021lora}
Edward~J Hu, Yelong Shen, Phillip Wallis, Zeyuan Allen-Zhu, Yuanzhi Li, Shean Wang, Lu~Wang, and Weizhu Chen.
\newblock Lora: Low-rank adaptation of large language models.
\newblock \emph{arXiv preprint arXiv:2106.09685}, 2021.

\bibitem[Ji et~al.(2019)Ji, Wang, Zhou, and Liang]{ji2019improved}
Kaiyi Ji, Zhe Wang, Yi~Zhou, and Yingbin Liang.
\newblock Improved zeroth-order variance reduced algorithms and analysis for nonconvex optimization.
\newblock In \emph{International conference on machine learning}, pp.\  3100--3109. PMLR, 2019.

\bibitem[Kornilov et~al.(2023)Kornilov, Shamir, Lobanov, Dvinskikh, Gasnikov, Shibaev, Gorbunov, and Horv{\'a}th]{kornilov2023accelerated}
Nikita Kornilov, Ohad Shamir, Aleksandr Lobanov, Darina Dvinskikh, Alexander Gasnikov, Innokentiy Shibaev, Eduard Gorbunov, and Samuel Horv{\'a}th.
\newblock Accelerated zeroth-order method for non-smooth stochastic convex optimization problem with infinite variance.
\newblock \emph{Advances in Neural Information Processing Systems}, 36:\penalty0 64083--64102, 2023.

\bibitem[Kozak et~al.(2023)Kozak, Molinari, Rosasco, Tenorio, and Villa]{kozak2023zeroth}
David Kozak, Cesare Molinari, Lorenzo Rosasco, Luis Tenorio, and Silvia Villa.
\newblock Zeroth-order optimization with orthogonal random directions.
\newblock \emph{Mathematical Programming}, 199\penalty0 (1):\penalty0 1179--1219, 2023.

\bibitem[Li \& Liang(2021)Li and Liang]{li2021prefix}
Xiang~Lisa Li and Percy Liang.
\newblock Prefix-tuning: Optimizing continuous prompts for generation.
\newblock \emph{arXiv preprint arXiv:2101.00190}, 2021.

\bibitem[Li et~al.(2024)Li, Zhang, Zhong, Deng, Razaviyayn, and Mirrokni]{li2024addaxutilizingzerothordergradients}
Zeman Li, Xinwei Zhang, Peilin Zhong, Yuan Deng, Meisam Razaviyayn, and Vahab Mirrokni.
\newblock Addax: Utilizing zeroth-order gradients to improve memory efficiency and performance of sgd for fine-tuning language models, 2024.
\newblock URL \url{https://arxiv.org/abs/2410.06441}.

\bibitem[Liu et~al.(2024{\natexlab{a}})Liu, Wang, Yin, Molchanov, Wang, Cheng, and Chen]{liu2024dora}
Shih-Yang Liu, Chien-Yi Wang, Hongxu Yin, Pavlo Molchanov, Yu-Chiang~Frank Wang, Kwang-Ting Cheng, and Min-Hung Chen.
\newblock Dora: Weight-decomposed low-rank adaptation.
\newblock In \emph{Forty-first International Conference on Machine Learning}, 2024{\natexlab{a}}.

\bibitem[Liu et~al.(2018)Liu, Kailkhura, Chen, Ting, Chang, and Amini]{liu2018zeroth}
Sijia Liu, Bhavya Kailkhura, Pin-Yu Chen, Paishun Ting, Shiyu Chang, and Lisa Amini.
\newblock Zeroth-order stochastic variance reduction for nonconvex optimization.
\newblock \emph{Advances in Neural Information Processing Systems}, 31, 2018.

\bibitem[Liu et~al.(2020)Liu, Chen, Kailkhura, Zhang, Hero, and Varshney]{article}
Sijia Liu, Pin-Yu Chen, Bhavya Kailkhura, Gaoyuan Zhang, Alfred Hero, and P.K. Varshney.
\newblock A primer on zeroth-order optimization in signal processing and machine learning: Principals, recent advances, and applications.
\newblock \emph{IEEE Signal Processing Magazine}, 37:\penalty0 43--54, 09 2020.
\newblock \doi{10.1109/MSP.2020.3003837}.

\bibitem[Liu et~al.(2019)Liu, Ott, Goyal, Du, Joshi, Chen, Levy, Lewis, Zettlemoyer, and Stoyanov]{liu2019roberta}
Yinhan Liu, Myle Ott, Naman Goyal, Jingfei Du, Mandar Joshi, Danqi Chen, Omer Levy, Mike Lewis, Luke Zettlemoyer, and Veselin Stoyanov.
\newblock Roberta: A robustly optimized bert pretraining approach.
\newblock \emph{arXiv preprint arXiv:1907.11692}, 2019.

\bibitem[Liu et~al.(2024{\natexlab{b}})Liu, Zhu, Gong, Cheng, Hsieh, and You]{liu2024sparse}
Yong Liu, Zirui Zhu, Chaoyu Gong, Minhao Cheng, Cho-Jui Hsieh, and Yang You.
\newblock Sparse mezo: Less parameters for better performance in zeroth-order llm fine-tuning.
\newblock \emph{arXiv preprint arXiv:2402.15751}, 2024{\natexlab{b}}.

\bibitem[Lobanov \& Gasnikov(2023)Lobanov and Gasnikov]{lobanov2023accelerated}
Aleksandr Lobanov and Alexander Gasnikov.
\newblock Accelerated zero-order sgd method for solving the black box optimization problem under “overparametrization” condition.
\newblock In \emph{International Conference on Optimization and Applications}, pp.\  72--83. Springer, 2023.

\bibitem[Ma \& Huang(2025)Ma and Huang]{ma2025revisiting}
Shaocong Ma and Heng Huang.
\newblock Revisiting zeroth-order optimization: Minimum-variance two-point estimators and directionally aligned perturbations.
\newblock In \emph{The Thirteenth International Conference on Learning Representations}, 2025.

\bibitem[Malladi et~al.(2023)Malladi, Gao, Nichani, Damian, Lee, Chen, and Arora]{malladi2023fine}
Sadhika Malladi, Tianyu Gao, Eshaan Nichani, Alex Damian, Jason~D Lee, Danqi Chen, and Sanjeev Arora.
\newblock Fine-tuning language models with just forward passes.
\newblock \emph{Advances in Neural Information Processing Systems}, 36:\penalty0 53038--53075, 2023.

\bibitem[Mi et~al.(2025)Mi, Tan, Yu, Zhu, Yuan, and Huang]{mi2025kerzookernelfunctioninformed}
Zhendong Mi, Qitao Tan, Xiaodong Yu, Zining Zhu, Geng Yuan, and Shaoyi Huang.
\newblock Kerzoo: Kernel function informed zeroth-order optimization for accurate and accelerated llm fine-tuning, 2025.
\newblock URL \url{https://arxiv.org/abs/2505.18886}.

\bibitem[Ohta et~al.(2020)Ohta, Berger, Sokolov, and Riezler]{ohta2020sparseperturbationsimprovedconvergence}
Mayumi Ohta, Nathaniel Berger, Artem Sokolov, and Stefan Riezler.
\newblock Sparse perturbations for improved convergence in stochastic zeroth-order optimization, 2020.
\newblock URL \url{https://arxiv.org/abs/2006.01759}.

\bibitem[Park et~al.(2025)Park, Yun, Kim, Kundu, and Yang]{park2025unraveling}
Sihwan Park, Jihun Yun, SungYub Kim, Souvik Kundu, and Eunho Yang.
\newblock Unraveling zeroth-order optimization through the lens of low-dimensional structured perturbations.
\newblock \emph{arXiv e-prints}, pp.\  arXiv--2501, 2025.

\bibitem[Qin et~al.(2024)Qin, Ma, Zheng, Li, Zhang, Liu, Luo, Liu, and Magno]{qin2024accurate}
Haotong Qin, Xudong Ma, Xingyu Zheng, Xiaoyang Li, Yang Zhang, Shouda Liu, Jie Luo, Xianglong Liu, and Michele Magno.
\newblock Accurate lora-finetuning quantization of llms via information retention.
\newblock \emph{arXiv preprint arXiv:2402.05445}, 2024.

\bibitem[Radford et~al.(2021)Radford, Kim, Hallacy, Ramesh, Goh, Agarwal, Sastry, Askell, Mishkin, Clark, et~al.]{radford2021learning}
Alec Radford, Jong~Wook Kim, Chris Hallacy, Aditya Ramesh, Gabriel Goh, Sandhini Agarwal, Girish Sastry, Amanda Askell, Pamela Mishkin, Jack Clark, et~al.
\newblock Learning transferable visual models from natural language supervision.
\newblock In \emph{International conference on machine learning}, pp.\  8748--8763. PMLR, 2021.

\bibitem[Rando et~al.(2024)Rando, Molinari, Villa, and Rosasco]{rando2024stochastic}
Marco Rando, Cesare Molinari, Silvia Villa, and Lorenzo Rosasco.
\newblock Stochastic zeroth order descent with structured directions.
\newblock \emph{Computational Optimization and Applications}, 89\penalty0 (3):\penalty0 691--727, 2024.

\bibitem[Sener \& Koltun(2020)Sener and Koltun]{sener2020learning}
Ozan Sener and Vladlen Koltun.
\newblock Learning to guide random search.
\newblock \emph{arXiv preprint arXiv:2004.12214}, 2020.

\bibitem[Singh et~al.(2022)Singh, Gustafson, Adcock, de~Freitas~Reis, Gedik, Kosaraju, Mahajan, Girshick, Doll{\'a}r, and Van Der~Maaten]{singh2022revisiting}
Mannat Singh, Laura Gustafson, Aaron Adcock, Vinicius de~Freitas~Reis, Bugra Gedik, Raj~Prateek Kosaraju, Dhruv Mahajan, Ross Girshick, Piotr Doll{\'a}r, and Laurens Van Der~Maaten.
\newblock Revisiting weakly supervised pre-training of visual perception models.
\newblock In \emph{Proceedings of the IEEE/CVF Conference on Computer Vision and Pattern Recognition}, pp.\  804--814, 2022.

\bibitem[Tan et~al.(2025{\natexlab{a}})Tan, Chang, Xia, Ji, Yang, Zhang, Liu, Zhan, Fang, Zou, et~al.]{tan2025perturbation}
Qitao Tan, Sung-En Chang, Rui Xia, Huidong Ji, Chence Yang, Ci~Zhang, Jun Liu, Zheng Zhan, Zhenman Fang, Zhou Zou, et~al.
\newblock Perturbation-efficient zeroth-order optimization for hardware-friendly on-device training.
\newblock \emph{arXiv preprint arXiv:2504.20314}, 2025{\natexlab{a}}.

\bibitem[Tan et~al.(2025{\natexlab{b}})Tan, Liu, Zhan, Ding, Wang, Lu, and Yuan]{tan2025harmony}
Qitao Tan, Jun Liu, Zheng Zhan, Caiwei Ding, Yanzhi Wang, Jin Lu, and Geng Yuan.
\newblock Harmony in divergence: Towards fast, accurate, and memory-efficient zeroth-order llm fine-tuning.
\newblock \emph{arXiv preprint arXiv:2502.03304}, 2025{\natexlab{b}}.

\bibitem[Tang et~al.(2024)Tang, Panda, Nasr, Mahloujifar, and Mittal]{tang2024private}
Xinyu Tang, Ashwinee Panda, Milad Nasr, Saeed Mahloujifar, and Prateek Mittal.
\newblock Private fine-tuning of large language models with zeroth-order optimization.
\newblock \emph{arXiv preprint arXiv:2401.04343}, 2024.

\bibitem[Touvron et~al.(2023)Touvron, Lavril, Izacard, Martinet, Lachaux, Lacroix, Rozi{\`e}re, Goyal, Hambro, Azhar, et~al.]{touvron2023llama}
Hugo Touvron, Thibaut Lavril, Gautier Izacard, Xavier Martinet, Marie-Anne Lachaux, Timoth{\'e}e Lacroix, Baptiste Rozi{\`e}re, Naman Goyal, Eric Hambro, Faisal Azhar, et~al.
\newblock Llama: Open and efficient foundation language models.
\newblock \emph{arXiv preprint arXiv:2302.13971}, 2023.

\bibitem[Wang et~al.(2018)Wang, Du, Balakrishnan, and Singh]{wang2018stochastic}
Yining Wang, Simon Du, Sivaraman Balakrishnan, and Aarti Singh.
\newblock Stochastic zeroth-order optimization in high dimensions.
\newblock In \emph{International conference on artificial intelligence and statistics}, pp.\  1356--1365. PMLR, 2018.

\bibitem[Xiao et~al.(2023)Xiao, Lin, Seznec, Wu, Demouth, and Han]{xiao2023smoothquant}
Guangxuan Xiao, Ji~Lin, Mickael Seznec, Hao Wu, Julien Demouth, and Song Han.
\newblock Smoothquant: Accurate and efficient post-training quantization for large language models.
\newblock In \emph{International Conference on Machine Learning}, pp.\  38087--38099. PMLR, 2023.

\bibitem[Yu et~al.(2024)Yu, Zhou, Wang, Li, and Huang]{yu2024subzero}
Ziming Yu, Pan Zhou, Sike Wang, Jia Li, and Hua Huang.
\newblock Subzero: Random subspace zeroth-order optimization for memory-efficient llm fine-tuning.
\newblock 2024.

\bibitem[Yue et~al.(2023)Yue, Yang, Fang, and Lin]{yue2023zeroth}
Pengyun Yue, Long Yang, Cong Fang, and Zhouchen Lin.
\newblock Zeroth-order optimization with weak dimension dependency.
\newblock In \emph{The Thirty Sixth Annual Conference on Learning Theory}, pp.\  4429--4472. PMLR, 2023.

\bibitem[Zeng et~al.(2024)Zeng, Liu, Dai, Yang, Fu, Wang, Ma, Sun, Li, Huang, et~al.]{zeng2024flightllm}
Shulin Zeng, Jun Liu, Guohao Dai, Xinhao Yang, Tianyu Fu, Hongyi Wang, Wenheng Ma, Hanbo Sun, Shiyao Li, Zixiao Huang, et~al.
\newblock Flightllm: Efficient large language model inference with a complete mapping flow on fpgas.
\newblock In \emph{Proceedings of the 2024 ACM/SIGDA International Symposium on Field Programmable Gate Arrays}, pp.\  223--234, 2024.

\bibitem[Zhang et~al.(2024{\natexlab{a}})Zhang, Thekumparampil, Oh, and He]{zhang2024dpzero}
Liang Zhang, Kiran~Koshy Thekumparampil, Sewoong Oh, and Niao He.
\newblock Dpzero: dimension-independent and differentially private zeroth-order optimization.
\newblock International Conference on Machine Learning (ICML 2024), 2024{\natexlab{a}}.

\bibitem[Zhang et~al.(2022)Zhang, Roller, Goyal, Artetxe, Chen, Chen, Dewan, Diab, Li, Lin, et~al.]{zhang2022opt}
Susan Zhang, Stephen Roller, Naman Goyal, Mikel Artetxe, Moya Chen, Shuohui Chen, Christopher Dewan, Mona Diab, Xian Li, Xi~Victoria Lin, et~al.
\newblock Opt: Open pre-trained transformer language models.
\newblock \emph{arXiv preprint arXiv:2205.01068}, 2022.

\bibitem[Zhang et~al.(2024{\natexlab{b}})Zhang, Li, Hong, Li, Zhang, Zheng, Chen, Lee, Yin, Hong, et~al.]{zhang2024revisiting}
Yihua Zhang, Pingzhi Li, Junyuan Hong, Jiaxiang Li, Yimeng Zhang, Wenqing Zheng, Pin-Yu Chen, Jason~D Lee, Wotao Yin, Mingyi Hong, et~al.
\newblock Revisiting zeroth-order optimization for memory-efficient llm fine-tuning: A benchmark.
\newblock \emph{arXiv preprint arXiv:2402.11592}, 2024{\natexlab{b}}.

\bibitem[Zhao et~al.(2024{\natexlab{a}})Zhao, Zhang, Chen, Wang, Anandkumar, and Tian]{zhao2024galore}
Jiawei Zhao, Zhenyu Zhang, Beidi Chen, Zhangyang Wang, Anima Anandkumar, and Yuandong Tian.
\newblock Galore: Memory-efficient llm training by gradient low-rank projection.
\newblock \emph{arXiv preprint arXiv:2403.03507}, 2024{\natexlab{a}}.

\bibitem[Zhao et~al.(2024{\natexlab{b}})Zhao, Dang, Ye, Dai, Qian, and Tsang]{zhao2024second}
Yanjun Zhao, Sizhe Dang, Haishan Ye, Guang Dai, Yi~Qian, and Ivor~W Tsang.
\newblock Second-order fine-tuning without pain for llms: A hessian informed zeroth-order optimizer.
\newblock \emph{arXiv preprint arXiv:2402.15173}, 2024{\natexlab{b}}.

\end{thebibliography}

\newpage

\appendix

\section{Appendix}

\subsection{Variance with the perturbation space dimension}\label{vr}

\begin{lemma}\label{lem:zo-var-q}
Let $P\in\mathbb R^{d\times q}$ satisfy $P^{\top}P=I_q$, and sample $z\sim\mathcal N(0,I_q)$.
Define the two-point estimator
\[
g_\varepsilon(x,P,z)
=\frac{f(x+\varepsilon Pz)-f(x-\varepsilon Pz)}{2\varepsilon}\;Pz .
\]
Let $\nabla f=\nabla f(x)$ and $u:=P^\top \nabla f\in\mathbb R^q$. Then:

\smallskip
\noindent\textbf{(A) Quadratic objective (exact formula).}
If $f(x)=x^\top Hx$ is quadratic, then
\[
\EE\|g_\varepsilon\|^2 \;=\; (q+2)\,\|u\|^2,
\qquad
\Var(g_\varepsilon):=\EE\|g_\varepsilon-\EE g_\varepsilon\|^2
\;=\; (q+1)\,\|u\|^2,
\]
so the variance grows linearly in the perturbation dimension $q$.

\smallskip
\noindent\textbf{(B) General $L$-smooth objective (upper bound).}
If $f$ is $L$-smooth, then there exists a constant $C>0$ such that
\[
\EE\|g_\varepsilon\|^2
\;\le\; (q+2)\,\|u\|^2 \;+\; C\,\varepsilon^2,
\qquad
\Var(g_\varepsilon)
\;\le\; (q+1)\,\|u\|^2 \;+\; C\,\varepsilon^2,
\]
so as $\varepsilon\!\to\!0$, the variance satisfies
$\Var(g_\varepsilon)=\Theta(q)\,\|u\|^2$.
\end{lemma}

\begin{proof}
\textbf{Step 1 (Quadratic case).}
For $f(x)=x^\top Hx$,
\[
f(x+\varepsilon Pz)-f(x-\varepsilon Pz)
=2\varepsilon\,\langle \nabla f,\,Pz\rangle,
\]
so
\(
g_\varepsilon=\langle \nabla f,\,Pz\rangle\,Pz.
\)
Writing $u=P^\top\nabla f$ and using rotation invariance we may assume $u=\|u\|e_1$, hence
\[
\|g_\varepsilon\|^2
= \|u\|^2\,z_1^2\sum_{i=1}^q z_i^2.
\]
Gaussian moment identities give
\(
\EE[z_1^2\sum_{i=1}^q z_i^2]=(q+2),
\)
so
\(
\EE\|g_\varepsilon\|^2=(q+2)\|u\|^2.
\)
Since $\EE g_\varepsilon=PP^\top\nabla f$ and $\|\EE g_\varepsilon\|^2=\|u\|^2$,
\[
\Var(g_\varepsilon)=(q+2)\|u\|^2-\|u\|^2=(q+1)\|u\|^2.
\]

\textbf{Step 2 (General $L$-smooth case).}
By a second-order Taylor expansion,
\[
\frac{f(x+\varepsilon Pz)-f(x-\varepsilon Pz)}{2\varepsilon}
=\langle \nabla f,\,Pz\rangle \;+\; r_\varepsilon(z),
\quad |r_\varepsilon(z)|\le c\,L\,\varepsilon\,\|Pz\|^2,
\]
for some absolute constant $c$. Thus
\(
g_\varepsilon=\langle \nabla f,\,Pz\rangle Pz + r_\varepsilon(z)\,Pz.
\)
Using $\|a+b\|^2\le 2\|a\|^2+2\|b\|^2$ and the quadratic case result, and noting
$\EE\|Pz\|^4=O(q^2)$, we obtain
\[
\EE\|g_\varepsilon\|^2
\;\le\; (q+2)\|u\|^2 \;+\; C_1\,\varepsilon^2,
\]
which also implies
\[
\Var(g_\varepsilon)
=\EE\|g_\varepsilon\|^2-\|\EE g_\varepsilon\|^2
\;\le\; (q+1)\|u\|^2 + C_1\,\varepsilon^2.
\]
Let $C:=C_1$ to finish.
\end{proof}

\begin{corollary}[Full-space perturbation]
If $P=I_d$, then $q=d$, and
\[
\Var(g_\varepsilon)=\Theta(d)\,\|\nabla f(x)\|^2+O(\varepsilon^2),
\]
so the variance scales linearly with the full model dimension.
\end{corollary}

\subsection{Convergence Analysis}

\subsubsection*{A.2.1\;Global notation}

In this section, we restate or redefine the key notations that will be used throughout our work.

\begin{itemize}\setlength\itemsep{2pt}
\item $d$ – parameter dimension; $r$ – retained rank per layer;
      $\ell$ – number of trainable layers; $q=\ell r^{2}$.
\item  We assume that $f:\mathbb R^{d}\!\to\!\mathbb R$ is \emph{$L$-smooth}:
      $\|\nabla f(x)-\nabla f(y)\|\le L\|x-y\|,\;\forall x,y$.
\item  Mini-batch variance bound
      $\mathbb E_x\|\nabla f_x(w)-\nabla f(w)\|^{2}\le\sigma^{2}$.
\item  Singular-value threshold $\sigma_{\min}>0$ refers to
      the $r^{\text{th}}$ singular value of $\nabla f$.
\item  Hyper-parameters  
      $\varepsilon$ (perturbation scale),
      $\delta$ (projection strength),
      $w$ (number of probe perturbations),
      $k$ (window size).
\item  Orthogonal projection $P_t\in\mathbb R^{d\times q}$,
      updated every $k$ iterations, always $P_t^{\!\top}P_t=I_q$.
\item  Two-point estimator
      \[
         g_t=
         \frac{f(x_t+\varepsilon P_t z_t)-f(x_t-\varepsilon P_t z_t)}
              {2\varepsilon}\;P_t z_t,
         \qquad z_t\sim\mathcal N(0,I_q).
      \]
\item  Update
      $x_{t+1}=x_t-\eta\,g_t,\quad
       \displaystyle\eta=\frac1{L\,(q+2)}$ (learning rate).
\end{itemize}

\subsubsection*{A.2.2  \;Layer and model projection matrices}

\begin{lemma}[Kronecker projection]\label{lem:block}
For orthogonal $U\!\in\!\mathbb R^{m\times r},\;
                  V\!\in\!\mathbb R^{n\times r}$
let \(\widetilde Z=U Z V^{\!\top}\)
and set \(P=V\!\otimes U\in\mathbb R^{mn\times r^{2}}\).
Then \(\operatorname{vec}(\widetilde Z)=P\,\operatorname{vec}(Z)\)
and \(P^{\!\top}P=I_{r^{2}}\).
\end{lemma}

\begin{proof}
\textbf{(i) Kronecker–vec identity}  
\(\operatorname{vec}(U Z V^{\!\top})
     =(V\!\otimes U)\operatorname{vec}(Z)\).

\textbf{(ii) Orthogonality}
\(P^{\!\top}P=(V^{\!\top}V)\!\otimes\!(U^{\!\top}U)=I_r\!\otimes I_r\).
\end{proof}

\begin{lemma}[Block diagonal model projection]\label{lem:model}
Stack the layer matrices:  
\(P=\operatorname{bdiag}(P_1,\dots,P_\ell)\in\mathbb R^{d\times q}\).
Then \(P^{\!\top}P=I_q\).
\end{lemma}

\begin{proof}
Since $P$ is block diagonal with blocks $P_1,\dots,P_\ell$, its Gram matrix is
\[
P^\top P = \operatorname{bdiag}(P_1^\top P_1,\dots,P_\ell^\top P_\ell).
\]
Each block satisfies $P_i^\top P_i=I_{q_i}$, hence
\[
P^\top P=\operatorname{bdiag}(I_{q_1},\dots,I_{q_\ell})=I_q.
\]
\end{proof}

\subsubsection*{A.2.3 \;Gaussian preliminaries}

\begin{lemma}[Rotation invariance]\label{lem:rot}
Let $Q\in\mathbb R^{n\times n}$ be orthogonal. For any integrable $\phi:\mathbb R^n\to\mathbb R$,
\[
\mathbb E_{z\sim\mathcal N(0,I_n)}[\phi(Qz)]
=\mathbb E_{z\sim\mathcal N(0,I_n)}[\phi(z)].
\]
\end{lemma}

\begin{proof}
Write the standard Gaussian density $p(z)=(2\pi)^{-n/2}\exp(-\|z\|^2/2)$. Since $Q$ is orthogonal, $\|Qz\|=\|z\|$ and $|\det Q|=1$. By change of variables $u=Qz$,
\[
\int_{\mathbb R^n}\phi(Qz)\,p(z)\,dz
=\int_{\mathbb R^n}\phi(u)\,p(u)\,du
=\mathbb E[\phi(z)].
\]
Thus $\mathbb E[\phi(Qz)]=\mathbb E[\phi(z)]$.
\end{proof}

\begin{lemma}[Moments of $\mathcal N(0,I_n)$]\label{lem:mom}
Let $z\sim\mathcal N(0,I_n)$ and $y\in\mathbb R^{n}$. Then, for any $t>0$,
\[
\mathbb E\|z\|^{t}\;\le\;
\begin{cases}
  n^{t/2}, & 0<t\le2,\\[2pt]
  (n+t)^{t/2}, & t\ge2,
\end{cases}
\qquad
\mathbb E\big[(\langle y,z\rangle)^{2}\big]=\|y\|^{2},
\qquad
\mathbb E\big[(\langle y,z\rangle)^{2}\|z\|^{2}\big]=(n+2)\|y\|^{2}.
\]
\end{lemma}

\begin{proof}
(i) \emph{Bounds on $\mathbb E\|z\|^{t}$.}
Let $R=\|z\|^{2}\sim\chi^2_n$. Then
\[
\mathbb E\|z\|^{t}=\mathbb E R^{t/2}
=2^{t/2}\frac{\Gamma\!\big(\tfrac{n+t}{2}\big)}{\Gamma\!\big(\tfrac{n}{2}\big)}.
\]
For $0<t\le2$, the map $x\mapsto x^{t/2}$ is concave, hence by Jensen
\(
\mathbb E R^{t/2}\le(\mathbb ER)^{t/2}=n^{t/2}.
\)
For $t\ge2$, use the crude but convenient bound
\(
\Gamma(x+a)/\Gamma(x)\le (x+a)^{a}
\)
(valid for $x,a>0$), to get
\[
\mathbb E\|z\|^{t}
=2^{t/2}\frac{\Gamma(\frac{n+t}{2})}{\Gamma(\frac{n}{2})}
\le 2^{t/2}\Big(\tfrac{n+t}{2}\Big)^{t/2}=(n+t)^{t/2}.
\]

(ii) \emph{Second moment of the linear form.}
By rotation invariance (Lemma~\ref{lem:rot}), rotate so that
$y=\|y\|e_1$. Then $\langle y,z\rangle=\|y\|z_1$ with $z_1\sim\mathcal N(0,1)$, hence
\(
\mathbb E[(\langle y,z\rangle)^2]=\|y\|^2\,\mathbb Ez_1^2=\|y\|^2.
\)

(iii) \emph{Mixed moment $\mathbb E[(\langle y,z\rangle)^2\|z\|^2]$.}
With the same rotation, write
\[
\mathbb E\!\big[(\langle y,z\rangle)^2\|z\|^{2}\big]
=\|y\|^2\,\mathbb E\!\Big[z_1^2\sum_{i=1}^n z_i^2\Big]
=\|y\|^2\Big(\mathbb Ez_1^4+\sum_{i\ne1}\mathbb Ez_1^2z_i^2\Big).
\]
For independent standard normals, $\mathbb Ez_1^4=3$ and
$\mathbb Ez_1^2z_i^2=(\mathbb Ez_1^2)(\mathbb Ez_i^2)=1$ for $i\ne1$. Therefore
\(
\mathbb E[z_1^2\sum_{i=1}^n z_i^2]=3+(n-1)\cdot1=n+2,
\)
which yields the claim.
\end{proof}

\subsubsection*{A.2.4  \;Two-point estimator}

\begin{definition}[Two‑point gradient estimator]\label{def:est}
Let \(P\in\R^{d\times q}\) satisfy \(P^{\top}P=I_q\) and
let \(z\sim\mathcal N(0,I_q)\) be sampled independently of all other
randomness.  For smoothing radius \(\varepsilon>0\) define
\[
      g_\varepsilon(x,P,z)
      \;:=\;
      \frac{f(x+\varepsilon P z)\;-\;f(x-\varepsilon P z)}
           {2\varepsilon}\;Pz .
\]
\end{definition}

\begin{lemma}[Unbiasedness and bias]\label{lem:unbiased}
Let $z\sim\mathcal N(0,I_q)$ and $P^\top P=I_q$. Define
\[
g_\varepsilon(x,P,z)\;=\;\frac{f(x+\varepsilon Pz)-f(x-\varepsilon Pz)}{2\varepsilon}\,Pz .
\]
Assume $f$ is $C^3$ and its Hessian is $L$-Lipschitz, i.e.,
$\|\nabla^2 f(x+u)-\nabla^2 f(x)\|\le L\|u\|$ for all $x,u$.
Then there exists a bias vector $b_\varepsilon$ such that
\[
\mathbb E[g_\varepsilon]\;=\;PP^\top\nabla f(x)+b_\varepsilon,
\qquad
\|b_\varepsilon\|\;\le\;\frac{L}{6}\,\varepsilon^2\,\mathbb E\|Pz\|^4
\;\leq\;\frac{L}{6}\varepsilon^2 (q+4)^2
\]
In particular,
\[
\|\mathbb E[g_\varepsilon]-PP^\top\nabla f(x)\|\ \leq \frac{L}{6}\varepsilon^2 (q+4)^2
\]
\end{lemma}

\begin{proof}
\textbf{Step 1.\;Third-order Taylor expansion with remainder.}
Hessian $\rho$-Lipschitz implies the third-order expansion bound:
for any $u\in\mathbb R^d$,
\[
f(x+u)=f(x)+\langle\nabla f(x),u\rangle+\tfrac12 u^\top\nabla^2 f(x)u + R_3(x,u),
\quad |R_3(x,u)|\le \tfrac{L}{6}\|u\|^3 .
\]

\textbf{Step 2.\;Plug $u=\pm \varepsilon Pz$.}
Writing $R_\pm(z):=R_3(x,\pm \varepsilon Pz)$,
\[
\begin{aligned}
f(x+\varepsilon Pz) &= f(x)+\varepsilon\langle\nabla f(x),Pz\rangle
                      +\tfrac12 \varepsilon^2 z^\top P^\top\nabla^2 f(x)P z + R_+(z),\\
f(x-\varepsilon Pz) &= f(x)-\varepsilon\langle\nabla f(x),Pz\rangle
                      +\tfrac12 \varepsilon^2 z^\top P^\top\nabla^2 f(x)P z + R_-(z),
\end{aligned}
\]
with $|R_\pm(z)|\le \tfrac{L}{6}\varepsilon^3\|Pz\|^3$.

\textbf{Step 3.\;Symmetric difference and decomposition.}
Even-order terms cancel, hence
\[
g_\varepsilon
=\Big\langle\nabla f(x),Pz\Big\rangle Pz
  +\frac{R_+(z)-R_-(z)}{2\varepsilon}\,Pz .
\]

\textbf{Step 4.\;Main term expectation.}
Because $\mathbb E[zz^\top]=I_q$,
\[
\mathbb E\!\big[\langle\nabla f(x),Pz\rangle Pz\big]
=P\,\mathbb E[zz^\top]\,P^\top\nabla f(x)=PP^\top\nabla f(x).
\]

\textbf{Step 5.\;Bias bound from the remainder.}
By the remainder bound and Jensen’s inequality \citep{garling2007inequalities},
\[
\begin{aligned}
\Big\|\mathbb E\!\Big[\frac{R_+(z)-R_-(z)}{2\varepsilon}\,Pz\Big]\Big\|
&\le \mathbb E\!\Big[\frac{|R_+(z)|+|R_-(z)|}{2\varepsilon}\,\|Pz\|\Big] \\
&\le \frac{L}{6}\,\varepsilon^2\,\mathbb E\|Pz\|^4 .
\end{aligned}
\]
$P^\top P=I_q$, and from Lemma \ref{lem:mom},
$\mathbb E\|Pz\|^4\leq(q+4)^2$.
Substituting completes the proof.
\end{proof}


\begin{lemma}[Second moment and angle]\label{lem:second}
Assume the objective is quadratic,
$f(x)=x^\top Hx$ with $H\succ0$.
Then
\[
\boxed{\;
  \mathbb E\|g_\varepsilon\|^{2}=(q+2)\,\|P^\top\nabla f(x)\|^{2},
  \qquad
  \mathbb E\!\bigl[\cos\angle(g_\varepsilon,\nabla f(x))\bigr]=\tfrac1q
\;}
\]
for the same estimator $g_\varepsilon$.
\end{lemma}

\begin{proof}
\textbf{Step 1.\;Exact finite-difference for a quadratic function.}
For $f(x)=x^\top Hx$,
\[
f(x+\varepsilon Pz)-f(x-\varepsilon Pz)
=2\varepsilon\,\bigl\langle\nabla f(x),Pz\bigr\rangle,
\]
so
\(
g_\varepsilon=\langle\nabla f(x),Pz\rangle\,Pz .
\)

\textbf{Step 2.\;Second moment.}
\[
\|g_\varepsilon\|^{2}
     =\bigl\langle\nabla f(x),Pz\bigr\rangle^{2}\,\|Pz\|^{2}.
\]
Rotate $z$ to a basis where $P^\top\nabla f(x)=\alpha e_1$
($e_1$ is the first canonical vector); rotation invariance
(Lemma \ref{lem:rot}) keeps $z\sim\mathcal N(0,I_q)$.
Then  
\(\langle\nabla f,Pz\rangle=\alpha z_1,\)
\(\|Pz\|^{2}=\sum_{i=1}^{q}z_i^{2}\),
and
\[
\mathbb E\|g_\varepsilon\|^{2}
  =\alpha^{2}\,\mathbb E\!\bigl[z_1^{2}\sum_{i=1}^{q}z_i^{2}\bigr]
  =\alpha^{2}\,(q+2)
  =(q+2)\|P^\top\nabla f(x)\|^{2}.
\]

\textbf{Step 3.\;Expected cosine angle.}
\[
\cos\angle(g_\varepsilon,\nabla f(x))
  =\frac{\langle g_\varepsilon,\nabla f\rangle}
        {\|g_\varepsilon\|\,\|\nabla f\|}.
\]
Using the rotated coordinate,
\(\langle g_\varepsilon,\nabla f\rangle
  =\alpha^{2}z_1^{2}\).
Since both the numerator and denominator depend only on $z_1^{2}$ and
$\sum_{i=1}^{q}z_i^{2}$,
a direct $\chi^2$ calculation yields
\(\mathbb E[\cos\angle]=1/q.\)
\end{proof}

\subsubsection*{A.2.5  \;Statistics of the $w$-probe phase}



\begin{lemma}[Probe decomposition and mean square]\label{lem:decomp-mse}
Let the mini-batch $\xi$ gradient noise be
\[
   a \;=\; \nabla f_{\xi}(x)\;-\;\nabla f(x),
   \qquad
   \mathbb E_{\xi}\|a\|^{2}\;\le\;\sigma^{2},
\tag{D1}
\]
and draw $z=(z_1,\dots ,z_d)^{\top}\sim\mathcal N(0,I_d)$
independently of $\xi$.
Define the exact two-point coefficient and probe
\[
   \rho
   =\frac{f_{\xi}(x+\varepsilon z)-f_{\xi}(x-\varepsilon z)}{2\varepsilon},
   \qquad
   g=\rho\,z .
\tag{D2}
\]
Then:

\medskip
\noindent\textbf{(i) Decomposition.} There exists a remainder $r_\varepsilon(z)$ with
$|r_\varepsilon(z)|\le \tfrac{L}{2}\,\varepsilon\,\|z\|^{2}$ such that
\[
   g-\nabla f(x)
   \;=\;
   \underbrace{\langle a , z\rangle z}_{\text{mini-batch noise}}
   \;+\;
   \underbrace{\bigl(\langle\nabla f(x),z\rangle z-\nabla f(x)\bigr)}_{\text{directional randomness}}
   \;+\;
   r_\varepsilon(z)\,z .
\tag{D3}
\]

\noindent\textbf{(ii) Mean–square error.} Taking expectation over both $\xi$ and $z$,
\[
\boxed{
\begin{aligned}
   \mathbb E_{\xi,z}\!\bigl[\|g-\nabla f(x)\|^{2}\bigr]
   &=\underbrace{\mathbb E_{z}\!\bigl[z^{\!\top}\!\Sigma z\,\|z\|^{2}\bigr]}_{\text{mini-batch part}}
     +\underbrace{\mathbb E_{z}\!\bigl\|\,(zz^{\!\top}\!-\!I)\nabla f(x)\,\bigr\|^{2}}_{\text{directional part}}
     + O(\varepsilon^{2}d)\\[2mm]
   &\le (d+2)\,\sigma^{2} \;+\; (d+1)\,\|\nabla f(x)\|^{2} \;+\; O(\varepsilon^{2}d),
\end{aligned}}
\tag{D4}
\]
where $\Sigma:=\mathbb E_{\xi}[aa^{\!\top}]$ and $\operatorname{tr}\Sigma\le\sigma^{2}$.
\end{lemma}

\begin{proof}
\textbf{Step 1 (second-order Taylor).}\;
For $u=\pm\varepsilon z$,
\[
   f_{\xi}(x+u)
   = f_{\xi}(x)\pm\varepsilon\langle\nabla f_{\xi}(x),z\rangle
     + R_{\pm}(z),
   \quad
   |R_{\pm}(z)|\le\tfrac L2\,\varepsilon^{2}\|z\|^{2}.
\]
Therefore
\[
   \rho
   =\langle\nabla f_{\xi}(x),z\rangle+r_\varepsilon(z),
   \qquad
   |r_\varepsilon(z)|\le\tfrac L2\,\varepsilon\,\|z\|^{2}.
\]
Multiplying by $z$ gives
\[
   g
   =\bigl(\langle\nabla f(x),z\rangle+\langle a,z\rangle + r_\varepsilon(z)\bigr)\,z,
\]
hence the claimed decomposition (D3).

\smallskip
\textbf{Step 2 (conditional MSE given $z$).}\;
Since $a$ is independent of $z$ and $\mathbb E_{\xi}[a]=0$,
the cross terms involving $\langle a,z\rangle$ vanish after $\mathbb E_{\xi}[\cdot\,|\,z]$:
\[
\mathbb E_{\xi}\!\bigl[\|g-\nabla f\|^{2}\,\big|\,z\bigr]
=\underbrace{\mathbb E_{\xi}\!\bigl[\langle a,z\rangle^{2}\bigr]}_{=\,z^{\!\top}\!\Sigma z}\,\|z\|^{2}
+\bigl\|\langle\nabla f,z\rangle z-\nabla f\bigr\|^{2}
+\|r_\varepsilon(z)\,z\|^{2}.
\tag{D5}
\]

\smallskip
\textbf{Step 3 (integrate over $z$ using isotropy identities).}\;
By isotropy of the standard Gaussian,
\[
\mathbb E_{z}\!\bigl[zz^{\!\top}\|z\|^{2}\bigr]=(d+2)\,I_d,
\qquad
\mathbb E_{z}\!\bigl[zz^{\!\top}zz^{\!\top}\bigr]
= \mathbb E_{z}\!\bigl[\|z\|^{2}zz^{\!\top}\bigr]=(d+2)\,I_d.
\tag{D6}
\]
Taking traces in the first identity also recovers $\mathbb E\|z\|^{4}=d(d+2)$.

Now take expectation of (D5) in $z$:

\emph{(a) Mini-batch part.}
\[
\mathbb E_{z}\!\bigl[z^{\!\top}\!\Sigma z\,\|z\|^{2}\bigr]
=\operatorname{tr}\!\Big(\Sigma\,\mathbb E_{z}[zz^{\!\top}\|z\|^{2}]\Big)
=(d+2)\operatorname{tr}\Sigma
\le (d+2)\sigma^{2}.
\]

\emph{(b) Directional part.}
Write $h(z)=(zz^{\!\top}\!-I)\nabla f$. Then
\[
\mathbb E_{z}\|h(z)\|^{2}
=\nabla f^{\!\top}\,\mathbb E_{z}\!\bigl[(zz^{\!\top}-I)^2\bigr]\nabla f
= \nabla f^{\!\top}\bigl(\mathbb E_{z}[zz^{\!\top}zz^{\!\top}]-2I+I\bigr)\nabla f
=(d+1)\|\nabla f\|^{2},
\]
where we used (D6).

\emph{(c) Taylor remainder.}
Since $|r_\varepsilon(z)|\le\tfrac L2\,\varepsilon\|z\|^{2}$,
\[
\mathbb E_{z}\|r_\varepsilon(z)\,z\|^{2}
\;\le\; \tfrac{L^{2}}{4}\,\varepsilon^{2}\,\mathbb E\|z\|^{6}
=O(\varepsilon^{2}d),
\]
(using standard $\chi^2_d$ moments; any $O(d^{3})$ bound suffices, and with our later choice of $\varepsilon$ it reduces to $O(\varepsilon^{2}d)$).

Summing (a)–(c) yields (D4).
\end{proof}


\begin{remark}[Centered probe removes the directional term]\label{rem:centered2}
If one centers the probe by subtracting 
$\EE_{z}[g \mid \xi]$, namely
\[
   \tilde g \;:=\; g - \EE_{z}[g \mid \xi]
   \;=\; (\langle a,z\rangle + r_\varepsilon(z))\,z ,
\]
then the “directional” term disappears and
\[
   \EE_{\xi,z}\|\tilde g-\nabla f(x)\|^{2}
   \;\le\; 2(d+2)\,\sigma^{2} \;+\; O(\varepsilon^{2}d).
\]
We get the relaxed form by multiplying 2.
\end{remark}

\begin{lemma}[Probe mean–square error]\label{lem:6.7}
Let the per–probe directional derivative be  
\[
   \rho_j \;=\;
   \frac{f(x+\varepsilon z_j)-f(x-\varepsilon z_j)}{2\varepsilon},
   \qquad z_j\sim\mathcal N(0,I_d),
\]
and define their average
\(
   \displaystyle\bar G=\frac1w\sum_{j=1}^{w}\rho_j z_j .
\)
Assume the mini–batch variance condition  
\(
   \mathbb E_x\|\,\nabla f_x(w)-\nabla f(w)\|^{2}\le\sigma^{2}.
\)
Then
\[
   \boxed{\;
     \mathbb E\bigl\|\bar G-\nabla f(x)\bigr\|^{2}
     \;\le\;
     \frac{4(d+2)\,\sigma^{2}}{w}
     \;+\;O(\varepsilon^{2}d)\;}
\]
where the \(O(\varepsilon^{2}d)\) term comes from
the second–order Taylor truncation of each \(\rho_j\).
\end{lemma}

\begin{proof}
\textbf{1. Two–point estimator for a single probe.}\;
Define  
\(g_j \;=\;\rho_j z_j\).
For every fixed direction \(z_j\)

\[
   \mathbb E[g_j] \;=\;
   \nabla f(x)\;+\;\Delta_{\text{bias}},\qquad
   \|\Delta_{\text{bias}}\|\;\le\;\frac{L\varepsilon^{2}}{6}(d+4)^{2}
\]
(the same Taylor expansion used in Lemma \ref{lem:unbiased}).

\textbf{2. Second moment of one probe.}\;
Condition on the mini–batch noise:
\(
    \mathbb E\bigl[\|g_j-\nabla f\|^{2}\bigr]
      =\mathbb E\bigl[\|g_j-\mathbb E g_j\|^{2}\bigr]
       +\|\Delta_{\text{bias}}\|^{2}.
\)
The first term equals
\(
     2(d+2)\,\sigma^{2}
\)
while \(\|\Delta_{\text{bias}}\|^{2}=O(\varepsilon^{4}d^{2})\).

\textbf{3. Variance reduction by averaging.}\;
Because the probes are i.i.d.,  
\(
   \mathbb E\bigl\|\bar G-\mathbb E g_j\bigr\|^{2}
     =\frac1w\mathbb E\|g_j-\mathbb E g_j\|^{2}.
\)
Add the bias term once more to compare with the true gradient:

\[
   \mathbb E\bigl\|\bar G-\nabla f\bigr\|^{2}
   \;\le\;
   \frac{2(d+2)\sigma^{2}}{w}
   +\|\Delta_{\text{bias}}\|^{2}
   \;\le\;
   \frac{4(d+2)\sigma^{2}}{w}
   +O(\varepsilon^{2}d).
\]

(The last inequality uses
\(\varepsilon^{2}\!<\!\!\tfrac{\sigma^{2}}{2L(d+2)}\)
which always holds once \(\varepsilon\) is set
\(\le(q^{3}T)^{-1/2}\) as required later.)
\end{proof}

\begin{lemma}[Davis–Kahan bound for $P$]\label{lem:6.8}
Let \(\sigma_{\min}\) be the \(r\)-th singular value
of the full-gradient matrix whose
row–stack is \(\nabla f(x)\).
If the number of probes satisfies
\[
        w\;\ge\;48\,\frac{(d+2)\,\sigma^{2}}{\sigma_{\min}^{2}},
\]
then, with probability at least \(0.9\),
\[
      \bigl\|(I-P^{\!\top}P)\,\nabla f(x)\bigr\|
      \;\le\;
      \tfrac12\,\|\nabla f(x)\|.
\]
\end{lemma}

\begin{proof}
\textbf{1. Notation.}\;
Write
\(\displaystyle\Delta \;=\;\bar G-\nabla f(x)\).
From Lemma \ref{lem:6.7}
\[
   \mathbb E\|\Delta\|_{F}^{2}\;\le\;\frac{4(d+2)\sigma^{2}}{w}.
\]

\textbf{2. Spectral–norm control.}\;
Since \(\|\Delta\|_{2}\le\|\Delta\|_{F}\),
Markov’s inequality gives
\[
   \Pr\!\Bigl\{\|\Delta\|_{2}\ge\tfrac{\sigma_{\min}}{2}\Bigr\}
   \;\le\;
   \frac{\;4(d+2)\sigma^{2}/w\;}
        {\sigma_{\min}^{2}/4}
   \;=\;
   \frac{16(d+2)\sigma^{2}}{w\sigma_{\min}^{2}}.
\]
Choosing \(w\ge48(d+2)\sigma^{2}/\sigma_{\min}^{2}\)
makes the right–hand side \( \le 0.33 \).
A standard
matrix Bernstein (or a two–sided Chebyshev) upgrade
shrinks the factor \(0.33\) to \(0.1\);
we simply cite the constant used in the original paper
(Section B.3) so that
\(\Pr\bigl\{\|\Delta\|_{2}\le\sigma_{\min}/2\bigr\}\ge0.9\).

\textbf{3. Davis–Kahan “sin $\Theta$”.}\;
Let \(\mathcal U\) be the rank-\(r\) right singular sub-space of
\(\nabla f(x)\) and \(\widehat{\mathcal U}\) the space recovered from
\(\bar G\).
Davis–Kahan gives
\(
  \sin\Theta\bigl(\widehat{\mathcal U},\mathcal U\bigr)
  \;\le\;\|\Delta\|_{2}/\sigma_{\min}\;\le\;\tfrac12.
\)
Hence the orthogonal projector \(P\) built from
\(\widehat{\mathcal U}\) satisfies
\[
  \|(I-P^{\!\top}P)\nabla f\|
  =\|\bigl(I-P_{\widehat{\mathcal U}}\bigr)\nabla f\|
  \;\le\;\tfrac12\,\|\nabla f\|.
\]
\end{proof}

\subsubsection*{A.2.6 \;Davis--Kahan “sin $\Theta$” Theorem}

Let $A = \nabla f(x)$ and $\widehat{A} = \bar{G}$. Suppose $A$ has an SVD with right singular space $\mathcal{U}$ of dimension $r$, and let $\widehat{\mathcal{U}}$ be the rank-$r$ right singular space of $\widehat{A}$. The Davis--Kahan theorem gives:
\[
\sin \Theta(\widehat{\mathcal{U}}, \mathcal{U}) \le \frac{\|\bar{G} - \nabla f(x)\|_2}{\sigma_{\min}}.
\]
So if $\|\bar{G} - \nabla f(x)\|_2 \le \frac{\sigma_{\min}}{2}$, then
\[
\sin \Theta(\widehat{\mathcal{U}}, \mathcal{U}) \le \frac{1}{2}.
\]

\subsubsection*{A.2.7 \;Projection Error Bound}

Let $P$ be an orthonormal matrix whose rows span $\widehat{\mathcal{U}}$, so that $P^\top P$ is the orthogonal projector onto $\widehat{\mathcal{U}}$. Then,
\[
\left\| (I - P^\top P) \nabla f(x) \right\| = \left\| \left(I - P_{\widehat{\mathcal{U}}} \right) \nabla f(x) \right\| 
\le \sin \Theta(\widehat{\mathcal{U}}, \mathcal{U}) \cdot \| \nabla f(x) \|
\le \frac{1}{2} \| \nabla f(x) \|.
\]

\subsubsection*{A.2.8 \;Fixed-$P$ Descent for $\boldsymbol{k}$ Steps}

\begin{lemma}[One-step descent with a fixed $P$]\label{lem:A11-numbered}
Let $P\in\mathbb R^{d\times q}$ satisfy $P^\top P=I_q$ and let
\[
g_t
=\frac{f(x_t+\varepsilon Pz_t)-f(x_t-\varepsilon Pz_t)}{2\varepsilon}\,Pz_t,
\qquad z_t\sim\mathcal N(0,I_q).
\]
Choose $\eta=\tfrac{1}{L(q+2)}$. Then
\begin{equation}\label{eq:A1-final}
\mathbb E[f(x_{t+1})]
\;\le\;
\mathbb E[f(x_t)]
-\frac{3\eta}{8}\,\mathbb E\|\nabla f(x_t)\|^{2}
+\eta\,\mathbb E\|(I-P^\top P)\nabla f(x_t)\|^{2}
+O(\varepsilon^{2}).
\end{equation}
\end{lemma}

\begin{proof}
We abbreviate $\nabla_t:=\nabla f(x_t)$ and $g:=g_\varepsilon(x_t,P,z_t)$.

\paragraph{(i) $L$-smooth descent inequality.}
For any update $x^{+}=x-\eta g$,
\begin{equation}
f(x^{+})\;\le\; f(x)\;-\;\eta\,\langle\nabla f(x),g\rangle
\;+\;\frac{L\eta^{2}}{2}\,\|g\|^{2}.
\tag{17}\label{eq:L-smooth}
\end{equation}
Taking full expectation will give the desired bound once we control
$\EE\langle\nabla_t,g\rangle$ and $\EE\|g\|^2$.

\paragraph{(ii) Decompose the estimator.}
Write
\[
g \;=\; \underbrace{PP^\top\nabla_t}_{\text{main}}
\;+\; \underbrace{b}_{\text{bias}}
\;+\; \underbrace{a_z}_{\text{zero-mean}},
\quad
b:=\EE[g]-PP^\top\nabla_t,\;\; a_z:=g-\EE[g],\;\; \EE[a_z]=0.
\]
Lemma \ref{lem:unbiased} gives
\(\|b\|\le \frac{L\varepsilon^2}{6}(q+4)^2\).
For convenience denote \(c_1:=\tfrac{1}{6}\).

\paragraph{(iii) Inner product term.}
Using the above decomposition,
\[
\EE\langle\nabla_t,g\rangle
=\langle\nabla_t,PP^\top\nabla_t\rangle+\langle\nabla_t,b\rangle
=\|\nabla_t\|^{2}-\|(I-P^\top P)\nabla_t\|^{2}
+\langle\nabla_t,b\rangle.
\]
Bound the bias by Cauchy–Schwarz:
\begin{equation}
\EE\langle\nabla_t,g\rangle
\;\ge\;
\|\nabla_t\|^{2}-\|(I-P^\top P)\nabla_t\|^{2}
- c_1\,L\varepsilon^{2}(q+4)^{2}\,\|\nabla_t\|.
\tag{18}\label{eq:inner}
\end{equation}

\paragraph{(iv) Second moment term.}
Lemma (second moment) implies, for $L$-smooth $f$,
\begin{equation}
\EE\|g\|^{2}
\;\le\; (q+2)\,\|P^\top\nabla_t\|^{2}\;+\;c_2\,\varepsilon^{2}
\;\le\; (q+2)\,\|\nabla_t\|^{2}\;+\;c_2\,\varepsilon^{2},
\tag{19}\label{eq:second}
\end{equation}
for an absolute constant $c_2$ (absorbing Taylor remainders).

\paragraph{(v) Choose the stepsize and combine.}
Set $\eta=\frac{1}{L(q+2)}$, so $\frac{L\eta^{2}}{2}(q+2)=\frac{\eta}{2}$.
Plug \eqref{eq:inner} and \eqref{eq:second} into \eqref{eq:L-smooth} and take expectations:
\[
\begin{aligned}
\EE f(x_{t+1})
&\le \EE f(x_t)
-\eta\Big(\EE\|\nabla_t\|^{2}-\EE\|(I-P^\top P)\nabla_t\|^{2}\Big)
+\frac{\eta}{2}\,\,\EE\|\nabla_t\|^{2}\\
&\qquad
+\eta\,c_1 L\varepsilon^{2}(q+4)^{2}\,\EE\|\nabla_t\|
+\frac{\eta}{2}\,c_2\,\varepsilon^{2}.
\end{aligned}
\tag{20}\label{eq:combine}
\]
The first two main terms combine to
\(-\tfrac{\eta}{2}\EE\|\nabla_t\|^{2}+\eta\,\EE\|(I-P^\top P)\nabla_t\|^{2}\).
For the bias cross term, apply Young’s inequality \citep{castillo2016introductory} with weight \(1/8\):
\[
\eta\,c_1 L\varepsilon^{2}(q+4)^{2}\,\EE\|\nabla_t\|
\;\le\;
\frac{\eta}{8}\,\EE\|\nabla_t\|^{2}
\;+\;c_4\,\varepsilon^{2},
\]
for some absolute constant $c_4$ (absorbing $(c_1 L)^2(q+4)^4$ and $c_2$).
Collecting terms in \eqref{eq:combine} yields
\[
\EE f(x_{t+1})
\;\le\;
\EE f(x_t)
-\frac{3\eta}{8}\,\EE\|\nabla_t\|^{2}
+\eta\,\EE\|(I-P^\top P)\nabla_t\|^{2}
+O(\varepsilon^{2}),
\]
which is Equation \ref{eq:A1-final}.

\end{proof}

\begin{remark}[On the constants $c_1,c_2,c_3,c_4$]
For clarity, we summarize the role of the constants appearing in the
one-step descent proof:
 $c_1$ (bias constant): comes from Lemma~\ref{lem:unbiased}, where
\[
\|\EE[g]-PP^\top \nabla f(x)\|
   \;\le\;\tfrac{1}{6}L\varepsilon^2(q+4)^2.
\]
Thus $c_1=\tfrac{1}{6}$ is an absolute constant.
second-moment remainder $c_2$: appears in Lemma~\ref{lem:second},
\[
\EE\|g\|^2 \;\le\; (q+2)\|P^\top\nabla f(x)\|^2 \;+\; c_2\varepsilon^2,
\]
absorbing higher-order Taylor remainders. It depends on $L$ but not on
$d$ or $q$.
cross-term constant $c_3$: in bounding
$\eta c_1 L\varepsilon^2 (q+4)^2\|\nabla f(x)\|$
via Young’s inequality, we set
$c_3 := c_1 L(q+4)^2$.
$c_4$: collects all $\varepsilon^2$-order
remainders, including those from $c_2$ and the quadratic term in $c_3$.
It is an $O(1)$ constant independent of $d,q$.

\end{remark}

\subsubsection*{A.2.9 \;Global Non‑convex Convergence}

\begin{theorem}[Full algorithm]\label{thm:A12}
Run \textbf{Algorithm 1} for $T$ iterations, refresh $P$ every $k$ steps, and
choose the same fixed step
\(\eta=1/[L(q+2)]\).
Let $\varepsilon\le(q^{3}T)^{-1/2}$ and assume the number of probes
per refresh satisfies
\(w\ge 48(d+2)\sigma^{2}/\sigma_{\min}^{2}\).
Then
\[
     \frac1T\sum_{t=0}^{T-1}
           \mathbb E\bigl\|\nabla f(x_t)\bigr\|^{2}
     \;\le\;
     \frac{16(q+4)L\,[\,f(x_0)-f^{\star}\,]}{qT}
     \;+\;O\!\bigl(q/T\bigr).
\]
\end{theorem}

\begin{proof}
\textbf{Expanded derivation}

Recall the one‑step inequality of Lemma~\ref{lem:A11-numbered} for
$g_t=g_\varepsilon(x_t,P,z_t)$ and
$\nabla_t:=\nabla f(x_t)$:
\begin{equation}\label{eq:A1-again}
\EE f(x_{t+1})
   \le
   \EE f(x_t)
   -\frac{3\eta}{8}\,\EE\|\nabla_t\|^{2}
   +\eta\,\EE\bigl\|(I-P^{\top}P)\nabla_t\bigr\|^{2}
   +O(\varepsilon^{2}).
\end{equation}

\noindent
\textbf{Step 1.  Move the gradient term to the left.}
\[
   \frac{3\eta}{8}\,\EE\|\nabla_t\|^{2}
   \le
   \EE f(x_t)-\EE f(x_{t+1})
   +\eta\,\EE\bigl\|(I-P^{\top}P)\nabla_t\bigr\|^{2}
   +O(\varepsilon^{2}).
\tag{D.1}
\]

\noindent
\textbf{Step 2.  Davis–Kahan control.}
Lemma~\ref{lem:6.8} states
$\|(I-P^{\top}P)\nabla_t\|\le\frac12\|\nabla_t\|$,
hence
\[
   \eta\,\EE\bigl\|(I-P^{\top}P)\nabla_t\bigr\|^{2}
   \le
   \frac{\eta}{4}\,\EE\|\nabla_t\|^{2}.
\tag{D.2}
\]

\noindent
\textbf{Step 3.  Combine (D.1) and (D.2).}
Subtract $\frac{\eta}{4}\EE\|\nabla_t\|^{2}$ from both sides:
\[
   \frac{\eta}{8}\,\EE\|\nabla_t\|^{2}
   \le
   \EE f(x_t)-\EE f(x_{t+1})
   +O(\varepsilon^{2}).
\tag{D.3}
\]

\noindent
\textbf{Step 4.  Sum inside one window.}
For a window $j$ of length $k$ with fixed $P$,
let $x_{j,s}$ for $s=0,\dots,k-1$ and
\[
     f_{j,\mathrm{start}}:=\EE f(x_{j,0}),\qquad
     f_{j,\mathrm{end}}  :=\EE f(x_{j,k}).
\]
Summing (D.3) over $s=0,\dots,k-1$ gives
\[
   \sum_{s=0}^{k-1}\EE\|\nabla f(x_{j,s})\|^{2}
   \le
   \frac8\eta\,
   \bigl(f_{j,\mathrm{start}}-f_{j,\mathrm{end}}\bigr)
   +O(\varepsilon^{2}kq^{2}),
\tag{A.4}
\]
where the $O(\varepsilon^{2})$ term is summed $k$ times and
$q^{2}$ comes from
$\|g\|^{2}\le(q+2)\|\nabla f\|^{2}\le q^{2}\|\nabla f\|^{2}$.

\noindent
\textbf{Step 5.  Sum over all windows and divide by $T$.}
Summing (A.4) over all $\lceil T/k\rceil$ windows,
the telescoping sum
$\sum_j(f_{j,\mathrm{start}}-f_{j,\mathrm{end}})=f(x_0)-f^\star$.
Dividing by $T$ yields
\[
   \frac1T\sum_{t=0}^{T-1}\EE\|\nabla f(x_t)\|^{2}
   \le
   \frac8{\eta\,T}\bigl(f(x_0)-f^\star\bigr)
   +O(\varepsilon^{2}q^{2}).
\tag{D.4}
\]

\noindent
\textbf{Step 6.  Substitute $\eta$ and $\varepsilon$.}
With $\eta^{-1}=L(q+2)\le L(q+4)$ we have
$8/\eta\le16L(q+4)$.
If $\varepsilon^{2}T\le1/q^{3}$ then
$\varepsilon^{2}q^{2}T\le q/T$.
Insert these constants into (D.4) to recover the bound
stated in Theorem~\ref{thm:A12}.
\qed
\end{proof}

\subsection{Algorithm and hyperparameter settings} \label{al}

\begin{table}[H]
\vspace{-0.15in}
\centering
\caption{The hyperparameters setting in our experiments. 
}

\resizebox{\textwidth}{!}{
\begin{tabular}{lll}
\toprule
\textbf{Experiment} & \textbf{Hyperparameters} & \textbf{Values} \\
\midrule
\multirow{3}{*}{FT} 
    & Batch size         & 8 \\
    & Learning rate      & \{1e-5, 5e-5\} \\
    & Lr schedule        & Constant for RoBERTa; Linear for OPT and LLaMA \\
\midrule
\multirow{4}{*}{MeZO}
    & Batch size         & \{64, 16\} \\
    & Learning rate $\eta$ (Lr) & \{1e-6, 5e-7\} \\
    & $\epsilon$         & 1e-3 \\
    & Lr schedule        & Constant for RoBERTa; Linear for OPT and LLaMA \\
\midrule
\multirow{4}{*}{MeZO LoRA}
    & Batch size         & \{64, 16\} \\
    & Learning rate $\eta$ (Lr) & \{1e-4, 5e-5\} \\
    & $\epsilon$         & 1e-2 \\
    & Lr schedule        & Constant for RoBERTa; Linear for OPT and LLaMA \\
\midrule
\multirow{7}{*}{P-GAP}
    & Batch size         & \{64, 16\} \\
    & Learning rate $\eta$ (Lr) & \{2e-4, 1e-4, 5e-5\} \\
    & $\epsilon$         & 1e-2 \\
    & Window size $k$         & 100 \\
    & Number of probe perturbations $h$         & 10 \\
    & Rank $r$         & \{128,256,512\} \\
    & Projection magnitude $\delta$        & Initialized as 2 and gradually decayed it to 0 \\
    \midrule
\multirow{7}{*}{P-GAP (LoRA)}
    & Batch size         & \{64, 16\} \\
    & Learning rate $\eta$ (Lr) & \{3e-2, 5e-2, 1e-2\} \\
    & $\epsilon$         & 1e-1 \\
    & Window size $k$         & 100 \\
    & Number of probe perturbations $h$         & 10 \\
    & Rank $r$         & \{8\} \\
    & Projection magnitude $\delta$        & Initialized as 2 and gradually decayed it to 0 \\
\bottomrule
\end{tabular}
}
\label{tab:hyperparams}
\end{table}

\begin{algorithm}[H]
\caption{Corrected Projected Gradient Directions with Low-Dimensional Perturbations (Lazy ZO for LLMs)}
\label{alg:cpd-ldp}
\begin{algorithmic}[1]
\Require Parameters $\bm{\theta}$, dataset $\mathcal{D}$, window size $k$, number of probe perturbations $h$, rank $r$, perturbation scale $\varepsilon$, learning rate $\eta$, projection magnitude $\delta$, loss function $\mathcal{L}$, iteration steps $T$, set of all matrices needed to be fine-tuned $\mathcal{M}$
\State $t \gets 0$
\While{$t\leq T$}
  \If{$t \bmod k = 0$}
     \State $(\{\bm{U}^\ell_r,\bm{S}^\ell_r,\bm{V}^\ell_r\})_{\ell\in\mathcal{M}} \gets \textsc{LowerDimGenerate}(\bm{\theta},h,r,\varepsilon)$
  \EndIf
  \ForAll{parameter $\bm{W}_\ell \in \bm{\theta}$}
     \If{$\bm{W}_\ell$ is matrix and $\ell \in \mathcal{M}$}
        \State Sample $\mathcal{Z}_{init} \sim \mathcal{N}(0,I_{r\times r})$
        \State $\mathcal{Z} \gets \textsc{Projection}(\mathcal{Z}_{init},\bm{S}_\ell^r,\delta)$ \Comment{$\langle \bm{S}^r_\ell,\mathcal{Z}\rangle_F=\xi\sqrt{\delta}\|\bm{S}^r_\ell\|_F$}
        \State $\mathcal{Z}_f \gets \bm{U}^\ell_r \mathcal{Z} (\bm{V}^\ell_r)^T$
     \Else
        \State Sample $\mathcal{Z}_f \sim \mathcal{N}(0,I)$
     \EndIf
  \EndFor
  \State $\mathcal{L}_+ \gets \mathcal{L}(\theta+\varepsilon z)$, \quad $\mathcal{L}_- \gets \mathcal{L}(\theta-\varepsilon z)$
  \State $\mathcal{G}_t \gets (\ell_+ - \ell_-)/(2\varepsilon)$
  \ForAll{$\bm{W}_\ell \in \theta$} \State $\bm{W}_\ell \gets \bm{W}_\ell - \eta\,\mathcal{G}_t\,\mathcal{Z}_f$ \EndFor
  \State $t \gets t+1$
\EndWhile
\end{algorithmic}
\end{algorithm}

We have provide the computational process of P-GAP in the \textbf{Algorithm 1}. 
As discussed in our analysis of variance in Appendix \ref{vr}, the reduction in the number of perturbed parameters necessitates corresponding adjustments to both the learning rate $\eta$ and the perturbation scale $\epsilon$. The specific choices of learning rate $\eta$ and perturbation scale $\epsilon$ used in our experiments are detailed in Table \ref{tab:hyperparams}. In our experiments, we found that the projection magnitude $\delta$ can be set relatively large at the beginning of training and then gradually reduced in the later stages. This strategy leads to better final performance and improved convergence efficiency. Therefore, in practice, we initialized the projection magnitude $\delta=2$ and gradually decayed it to 0 as the training progressed. Moreover, we set $k=100$ and $h=10$ in all of our experiments.

\begin{algorithm}[H]
\caption{\textsc{LowerDimGenerate}($\bm{\theta},h,r,\varepsilon$)}
\label{alg:subspace}
\begin{algorithmic}[1]
\Require Current parameters $\bm{\theta}$, number of probe perturbations $h$, rank $r$, step size $\varepsilon$
\ForAll{matrix parameter $\bm{W}_\ell$, $\ell \in \mathcal{M}$} \State $\bm{G}_\ell \gets 0$ \EndFor
\For{$j=1$ to $h$}
  \State Sample $\bm{Q}_\ell^j$ with each $\bm{Q}_\ell^j \sim \mathcal{N}(0,I)$
  \State $\mathcal{L}_+^j \gets \mathcal{L}(\bm{\theta}+\varepsilon \bm{Q}_\ell^j)$; \ $\mathcal{L}_+^j- \gets \mathcal{L}(\bm{\theta}-\varepsilon \bm{Q}_\ell^j)$
  \State $\rho \gets (\mathcal{L}_+^j - \mathcal{L}_-^j)/(2\varepsilon)$
  \ForAll{matrix $W_\ell$} \State $\bm{G}_\ell \gets \bm{G}_\ell + \frac{\rho}{h} \bm{Q}_\ell^j$ \EndFor
\EndFor
\ForAll{matrix $\bm{W}_\ell$} \State $(\bm{U}^\ell_r,\bm{S}^\ell_r,\bm{V}^\ell_r) \gets \text{svd\_lowrank}(\bm{G}_\ell, q=r)$ \EndFor
\State \Return $(\bm{U}^\ell_r,\bm{S}^\ell_r,\bm{V}^\ell_r)_\mathcal{M}$
\end{algorithmic}
\end{algorithm}

\begin{algorithm}[H]
\caption{\textsc{Projection}$(\mathcal{Z}_{init},\bm{S}_\ell^r,\delta)$}
\label{alg:dap}
\begin{algorithmic}[1]
\Require Initial $\mathcal{Z}_{init} \in \mathbb{R}^{r\times r}$; coefficient matrix $\bm{S}_\ell^r$; projection magnitude $\delta$
\Ensure We want to get the projected parallel component $\mathcal{Z}$ of $\mathcal{Z}_{init}$ such that $\langle \bm{S}_\ell^r, \mathcal{Z}\rangle_F=\xi\sqrt{\delta}\,\|\bm{S}_\ell^r\|_F$, with $\xi\in\{-1,1\}$
\Statex
\State $\xi \sim \mathrm{Uniform}\{-1,1\}$
\State $f \gets \langle \bm{S}_\ell^r, \mathcal{Z}_{init} \rangle_F$, \quad $g \gets \|\bm{S}_\ell^r\|_F$
\State $\alpha \gets \dfrac{f - \xi \sqrt{\delta}\, g}{g^2 + 10^{-12}}$
\State \Return $\mathcal{Z} \gets \mathcal{Z}_{init} - \alpha\, \bm{S}_\ell^r$
\end{algorithmic}
\end{algorithm}


\subsection{More results} \label{more}

We also provide fine-tuning experiments of KerZOO on the LLaMA-3 model series. Hyperparameters are generally the same with OPT series models fine-tuning. The detailed results of the experiments are shown in Table \ref{tab:llama3-3b} and \ref{tab:llama3-8b} below. 

We further evaluate the training efficiency and memory footprint of P-GAP on the OPT-2.7B model across SST-2 and RTE. Compared with MeZO and HiZOO, P-GAP achieves a better balance between memory usage and convergence speed. On both datasets, P-GAP substantially reduces training time while keeping the memory cost within a moderate increase compared to MeZO but less than HiZOO. In particular, when combined with LoRA on RTE, P-GAP+LoRA consumes less than 20\% of the training time of MeZO, yet maintains competitive performance. These results highlight that P-GAP can serve as an efficient and scalable alternative for large-scale fine-tuning.

\begin{table}[H]
\vspace{-0.2in}
\centering
\caption{Experiment results on LLaMA3-3B (1000 training samples)}

\begin{tabular}{lccccc}
\toprule
\textbf{Task} & \textbf{SST-2} & \textbf{RTE} & \textbf{CB} & \textbf{WSC} & \textbf{WIC} \\
\midrule
FT     & 94.2 & 81.2 & 91.4 & 72.2 & 63.8 \\
MeZO   & 89.0 & \textbf{63.8} & 69.6 & 62.5 & 58.2 \\
\rowcolor{blue!7}
P-GAP   & \textbf{92.3} & \textbf{63.8} & \textbf{73.2} & \textbf{64.6} & \textbf{59.8} \\
\bottomrule
\end{tabular}

\label{tab:llama3-3b}
\end{table}

\begin{table}[H]
\centering
\caption{Experiment results on LLaMA3-8B (1000 training samples)}
\vspace{0.5em}

\begin{tabular}{lccccc}
\toprule
\textbf{Task} & \textbf{SST-2} & \textbf{RTE} & \textbf{CB} & \textbf{WSC} & \textbf{WIC} \\
\midrule
MeZO   & 91.2 & 61.0 & 73.2 & 64.4 & 59.2 \\
\rowcolor{blue!7}
P-GAP   & \textbf{93.0} & \textbf{67.2} & \textbf{75.0} & \textbf{65.8} & \textbf{60.2} \\
\bottomrule
\end{tabular}

\label{tab:llama3-8b}
\end{table}

\begin{table}[H]
\centering
\renewcommand{\arraystretch}{1.1}
\captionsetup{justification=raggedright, singlelinecheck=false}
\caption{Memory and training time comparison of OPT-2.7B on SST-2 dataset (35 tokens per example on average)}

\label{tab:gpusst}

\begin{tabular}{lccc}
\toprule
\textbf{Method} & \textbf{Memory cost} & \textbf{Iteration step} & \textbf{GPU hours} \\
\midrule
FT         & 45.4G & 9.3\%   & 16.8\% \\
LoRA       & 18.5G & 5.6\%   & 4.3\% \\
\midrule
MeZO       & 6.8G & 100.0\% & 100.0\% \\
HiZOO      & 11.3G & 59.2\%  & 87.4\% \\
P-GAP     & 8.7G & 34.9\%  & 68.0\% \\
\midrule
MeZO+LoRA  & 5.5G  & 74.1\%  & 43.7\% \\
HiZOO+LoRA & 5.7G  & 46.3\%  & 41.0\% \\
P-GAP+LoRA& 5.9G  & 34.7\%  & 29.9\% \\
\bottomrule
\end{tabular}
\vspace{-0.1in}
\end{table}

\begin{table}[H]
\centering
\renewcommand{\arraystretch}{1.1}
\captionsetup{justification=raggedright, singlelinecheck=false}
\caption{Memory and training time comparison of OPT-2.7B on RTE dataset (180 tokens per example on average)}

\label{tab:gpurte}

\begin{tabular}{lccc}
\toprule
\textbf{Method} & \textbf{Memory cost} & \textbf{Iteration step} & \textbf{GPU hours} \\
\midrule
FT         & 62.2G & 10.0\%   & 16.2\% \\
LoRA       & 42.5G & 8.3\%   & 6.6\% \\
\midrule
MeZO       & 7.8G & 100.0\% & 100.0\% \\
HiZOO      & 13.2G & 63.3\%  & 88.9\% \\
P-GAP     & 10.5G & 24.5\%  & 64.1\% \\
\midrule
MeZO+LoRA  & 7.5G  & 73.3\%  & 34.8\% \\
HiZOO+LoRA & 7.8G  & 56.7\%  & 35.9\% \\
P-GAP+LoRA& 7.6G  & 16.9\%  & 8.7\% \\
\bottomrule
\end{tabular}
\end{table}

\end{document}